\documentclass[letterpaper,10pt]{article}
\usepackage{fullpage}
\usepackage{times}  
\usepackage{helvet}  
\usepackage{courier}  
\usepackage[hyphens]{url}  
\usepackage{graphicx} 
\usepackage{caption} 
\usepackage[colorlinks=true, citecolor=blue]{hyperref}

\usepackage{color}
\usepackage{amsfonts,amssymb,amsbsy,amsmath, amsthm}
\usepackage{empheq}
\usepackage{cleveref}
\usepackage{xcolor}
\usepackage{tikz}
\usetikzlibrary{shadows,arrows.meta,positioning,backgrounds,fit,chains,scopes}
\usepackage{caption}
\usepackage{subcaption}
\usepackage{enumitem}

\usepackage[noend,ruled]{algorithm2e}
\PassOptionsToPackage{ruled}{algorithm2e}

\RestyleAlgo{ruled}

\SetCommentSty{mycommfont}
 
\SetAlFnt{\small}
\SetAlCapFnt{\small}
\SetAlCapNameFnt{\small}
\SetAlCapHSkip{0pt}
\IncMargin{-\parindent}

\SetKwInput{KwInput}{Input}
\SetKwInput{KwOutput}{Output} 
\newtheorem{definition}{Definition}

\newcommand{\bm}[1]{\boldsymbol{#1}}
\newcommand{\argmax}{\mathrm{argmax}}

\newcommand{\one}{\mathrm{1}}
\newcommand{\set}[1]{\left\{ #1 \right\}}
\newcommand{\abs}[1]{\left| #1 \right|}
\newcommand{\norm}[1]{\left \lVert #1 \right \rVert}
\renewcommand{\bar}{\overline}
\renewcommand{\tilde}{\widetilde}
\newcommand{\E}{\mathbb{E}}

\newcommand{\calS}{\mathcal{S}}
\newcommand{\calR}{\mathcal{R}}

\newcommand{\calD}{\mathcal{D}}
\newcommand{\calH}{\mathcal{H}}

\newcommand{\calM}{\mathcal{M}}
\newcommand{\calA}{\mathcal{A}}
\newcommand{\Nm}{\mathbb{N}}

\newcommand{\R}{\mathbb{R}}

\newcommand{\calF}{\mathcal{F}}

\newcommand{\reg}{\textrm{Reg}}
\newcommand{\rew}{\textrm{Rew}}
\newcommand{\bgamma}{\boldsymbol{\gamma}}

\newcommand{\Pm}{\mathbb{P}}
\newcommand{\Vm}{\mathbb{V}}
\newcommand{\ucb}{\textrm{UCB}}
\newcommand{\var}{\textrm{Var}}
\newcommand{\ph}{\mathtt{P}_{\mathtt{H}} }
\newcommand{\xkh}{x_{k,h}}
\newcommand{\akh}{a_{k,h}}
\newcommand{\tdkh}{\tilde{\delta}_{k,h}}
\newcommand{\eps}{\varepsilon}
\newcommand{\hbgamma}{\widehat{\bgamma}}

\newtheorem{theorem}{Theorem}
\newtheorem{lemma}{Lemma}

\newtheorem{corollary}[theorem]{Corollary}

\newcount\Comments  
\newcommand{\kibitz}[2]{\ifnum\Comments=1{\color{#1}{#2}}\fi}

\usepackage[maxnames=10,backref=true,style=alphabetic]{biblatex}
\newcommand{\citet}[1]{\citeauthor*{#1}~\cite{#1}}
\title{Online Reinforcement Learning with Uncertain Episode Lengths}

\author{
Debmalya Mandal\\
Max Planck Institute for Software Systems\\
 \texttt{dmandal@mpi-sws.org}\\ 
 \and Goran Radanovi\'{c}\\
 Max Planck Institute for Software Systems\\
 \texttt{gradanovic@mpi-sws.org}\\
 \and Jiarui Gan\\
 University of Oxford\\
 \texttt{jiarui.gan@cs.ox.ac.uk}\\
\and Adish Singla\\
Max Planck Institute for Software Systems\\
 \texttt{adishs@mpi-sws.org}\\ 
 \and Rupak Majumdar\\
 Max Planck Institute for Software Systems\\
 \texttt{rupak@mpi-sws.org} 
 }

\begin{document}

\maketitle


\begin{abstract}
 Existing episodic reinforcement algorithms assume that the  length of an episode is fixed across time and known a priori. In this paper, we consider a general framework of  episodic reinforcement learning when the length of each episode is drawn from a distribution. We first establish that this problem is equivalent to online reinforcement learning with general discounting where the learner is trying to optimize the expected discounted sum of rewards over an infinite horizon, but where the discounting function is not necessarily geometric. We show that minimizing regret with this new general discounting is equivalent to minimizing regret with uncertain episode lengths.
 We then design a reinforcement learning algorithm that minimizes regret with general discounting but acts for the setting with uncertain episode lengths. We instantiate  our general bound for different types of discounting, including geometric and polynomial discounting. We also show that we can obtain similar regret bounds even when the uncertainty over the episode lengths is unknown, by estimating the unknown distribution over time. Finally, we compare our learning algorithms with existing value-iteration based episodic RL algorithms on a grid-world environment.
\end{abstract}

\section{Introduction}
We consider the problem of \emph{episodic reinforcement learning}, where a learning agent interacts with the environment over a number of episodes~\cite{SB18}. The framework of episodic reinforcement learning usually considers two types of episode lengths: either each episode has a fixed and invariant length $H$,
or each episode may have a varying length controlled by the learner.
The fixed-length assumption is relevant for recommender systems~\cite{Charu16} where the platform interacts with a user for a fixed number of rounds. 
Variable length episodes arise naturally in robotics~\cite{KBP13}, where each episode is associated with a learning agent completing a task, and so the length of the episode is entirely controlled by the learner.
Fixed horizon lengths make the design of learning algorithms easier, 
and is the usual assumption in most papers on theoretical reinforcement learning~\cite{AOM17, JZBJ18}. 

In this paper, we take a different perspective on episodic reinforcement learning and assume that the length of each episode is drawn from a distribution. This situation often arises in online platforms where the length of an episode (i.e., the duration of a visit by a user) is not fixed a priori, but follows a predictable distribution~\cite{OFJB14}. Additionally, various econometric and psychological evidence suggest that humans learn by maintaining a risk/hazard distribution over the future~\cite{Sozou98}, which can be interpreted as a distribution over the horizon length.
Despite a large and growing literature on episodic reinforcement learning, except for \cite{FGBB+19}, uncertain epsiodic lengths or settings with general survival rates of agents have not been studied before.

\textbf{Our Contributions}: In this paper, we describe reinforcement learning algorithms for general distributions over episode lengths. Our main contribution is a general learning algorithm which can be adapted to a given distribution over episode lengths to obtain sub-linear regret over time. In particular, our contributions are the following.

\begin{itemize}
    \item  We first establish an equivalence between maximization of expected total reward with uncertain episode lengths and maximization of expected (general) discounted sum of rewards over an infinite horizon. In particular, we show that minimization of regret  is equivalent in these two environments.
    \item 
    Next we design a learning algorithm for the setting with arbitrary distribution over the episode lengths. Our algorithm generalizes the value-iteration based learning algorithm of \citet{AOM17} by carefully choosing an effective horizon length and then updating the backward induction step based on the distribution over episode lengths. In order to analyze its regret, we use the equivalence result above, and bound its regret for a setting with general discounting.
    \item We instantiate our general regret bound for different types of discounting (or equivalently episode distributions), including geometric and polynomial discounting, and obtain sub-linear regret bounds. For geometric discounting with parameter $\gamma$, we bound regret by $\tilde{O}(\sqrt{SAT}/(1-\gamma)^{1.5})$ which matches the recently established minimax optimal regret for the non-episodic setting~\cite{HZG21}. For the polynomial discounting of the form $h^{-p}$ we upper bound regret by $\tilde{O}(\sqrt{SA}T^{\frac{1}{2-1/p}})$.
    \item Finally, we show that we can obtain similar regret bounds even when the uncertainty over the episode lengths is unknown, by estimating the unknown distribution over time. In fact, for geometric discounting, we  recover the same regret bound (i.e. $\tilde{O}(\sqrt{SAT}/(1-\gamma)^{1.5}$) up to logarithmic factors, and for the polynomial discounting we obtain a regret bound of $\tilde{O}(\sqrt{SA}T^{\frac{p}{1+2p}})$, which asymptotically matches the previous regret bound.
\end{itemize}


Our results require novel and non-trivial generalizations of episodic learning algorithms and straightforward extensions to existing algorithms do not work.
Indeed, a naive approach would be to use the expected episode length as the fixed horizon length $H$. 
However, this fails with heavy-tailed distributions which often appear in practice. 
Alternately, we could compute an upper bound on the episode length so that with high probability the lengths of all the $T$ episodes are within this bound. Such an upper bound can be computed with the knowledge of distribution over episode lengths and using standard concentration inequalities. 
However, these upper bounds become loose either with a large number of episodes or for heavy-tailed distributions.

\subsection{Related Work}
\textbf{Episodic Reinforcement Learning}: Our work is closely related to the UCB-VI algorithm of \citet{AOM17}, which achieves $O(\sqrt{HSAT})$ regret for episodic RL with fixed horizon length $H$. The main difference between our algorithm and UCB-VI is that we use a different equation for backward-induction where future payoffs are discounted by a factor of $\gamma(h+1) / \gamma(h)$ at step $h$, where $\gamma$ is a general discount function. Beyond \cite{AOM17}, several papers have considered different versions of episodic RL including changing transition function~\cite{JZBJ18, JL20}, and function approximation~\cite{JYWJ20, WSY20, YW20}.

\textbf{General Discounting}: Our work is also closely related to reinforcement learning with general discounting. Even though geometric discounting is the most-studied discounting because of its theoretical properties~\cite{Bertsekas12}, there is a wealth of evidence suggesting that humans use general discounting and time-inconsistent decision making
\cite{Ainslie75, Mazur85, GLM04}. 
In general, optimizing discounted sum of rewards with respect to a general discounting might be difficult as we are not guaranteed to have a stationary optimal policy. \citet{FGBB+19} study RL with hyperbolic discounting and learn many $Q$-values each with a different (geometric) discounting. 
Our model is more general, and our algorithm is based on a modified value iteration. We also obtain theoretical bounds on regret in our general setting.
Finally, \citet{Pitis19} introduced more general state, action based discounting but that is out of scope of this paper.

\textbf{Stochastic Shortest Path}: Our work is related to the stochastic shortest path (SSP), introduced by \citet{BT91}. In SSP, the goal of the learner is to reach a designated state in an MDP, and minimize the expected total cost of the policy before reaching that goal. Recently, there has been a surge of interest in deriving online learning algorithms for SSP~\cite{RCMK20, CEMR21, TZDP+21}.  Our setting differs from SSP in two ways. First, the horizon length is effectively controlled by the learner in SSP, once she has a good approximation of the model. But in our setting, the horizon length is drawn from a distribution at the start of an episode by the nature, and is unknown to the learner during that episode. Second, when the model is known in SSP, different policies induce different distributions over the horizon length. Therefore, in contrast to our setting, minimizing regret in SSP is not the same as minimizing regret under general discounting. 

\textbf{Other Related Work}: Note that uncertainty over episode lengths can also be interpreted as hazardous MDP~\cite{HM72}, where hazard rate is defined to be the negative rate of change of log-survival time. \citet{Sozou98} showed that different prior belief over hazard rates imply different types of discounting. We actually show equivalence between general discounting and uncertain episode lengths, even in terms of regret bounds. Finally, this setting is captured by the partially observable Markov decision processes~\cite{KLC98}, where one can make the uncertain parameters hidden and/or partially observable. 
%

\section{Model}\label{sec:model}
We consider the problem of episodic reinforcement learning with uncertain episode length. An agent interacts with an MDP $\calM = (S, \calA, r, \Pm, \ph)$, where $\ph$ denotes the probability distribution over the episode length. We assume that the rewards are bounded between $0$ and $1$. The agent interacts with the environment for $T$ episodes as follows.
\begin{itemize}
\item At episode $k \in [T]$, the starting state $x_{k,1}$ is chosen arbitrarily and the length of the episode $H_k\sim \ph(\cdot)$. \footnote{The parameter $H_k$ is unknown to the learner during episode $k$.}
\item For $h \in [H_k]$, let the state visited be $x_{k,h}$ and the action taken be $a_{k,h}$. Then, the next state $x_{k,h+1} \sim \Pm(\cdot | x_{k,h}, a_{k,h})$.
\end{itemize}

 The agent interacts with the MDP $\calM$ for $T$ episodes and the goal is to maximize the expected undiscounted sum of rewards. Given a sequence of $k$ episode lengths $\{H_k\}_{k \in [T]}$ the expected cumulative reward of an agent's policy $\bm{\pi} = \{\pi_k\}_{k \in [T]}$ is given as
 \begin{equation*}
\rew\left(\bm{\pi}; \{H_k\}_{k\in[T]} \right) = \sum_{k=1}^T \E\left[ \sum_{h=1}^{H_k} r(\xkh, \akh) \right]
 \end{equation*}
 Since each $H_k$ is a random variable drawn from the distribution $\ph(\cdot)$, we are interested in expected reward with respect to distribution $\ph$.
\begin{align}
&\E\left[\rew\left(\bm{\pi}; \{H_k\}_{k\in[T]} \right) \right] \nonumber \\&= \E\left[ \sum_{k=1}^T \sum_{H_k=1}^\infty \ph(H_k) \sum_{h=1}^{H_k} r(x_{k,h}, a_{k,h}) \right] \nonumber \\
&= \E_{\pi}\left[ \sum_{k=1}^T \sum_{h=1}^\infty \ph(H \ge h) r(x_{t,h}, a_{t,h}) \right]\label{eq:original-undiscounted-reward}
\end{align}
As is standard in the literature on online learning, we will consider the problem of minimizing regret instead of maximizing the reward. Given an episode length $H_k$ and starting state $x_{k,1}$ let $\pi^\star_k$ be the policy that maximizes the expected sum of rewards over $H_k$ steps i.e.
$
\pi^\star_k \in \argmax_{\pi} \E_{\pi}\left[\sum_{h=1}^{H_k} r(\xkh, \akh) \lvert x_{k,1} \right].
$
We will  write $V^{\pi_k}(x_{k,1}; H_k)$ to write the (undiscounted) value function of a policy $\pi_k$ over $H_k$ steps starting from state $x_{k,1}$. Then $\pi^\star_k$ is also defined as $\pi^\star_k \in \argmax_\pi V^\pi(x_{k,1}; H_k)$. We will also write $V^\star(x_{k,1}; H_k)$ to denote the corresponding value of the optimal value function. Now we can define the regret  over $T$ steps as follows.
\begin{definition}
The regret of a learning algorithm $\bm{\pi} = \{\pi_k\}_{k \in [T]}$ over $T$ steps with episode lengths $\{H_k\}_{k \in [T]}$ is
\begin{equation}
\label{defn:regret-original-random}
\reg\left(\bm{\pi}; \{H_k\}\right) = \sum_{k \in [T]} V^\star(x_{k,1}; H_k) - V^{\pi_k}(x_{k,1}; H_k)
\end{equation}
\end{definition}
Note that the regret as defined in \cref{defn:regret-original-random} is actually a random variable as the episode lengths are also randomly generated from the distribution $\ph(\cdot)$. So we will be interested in bounding the expected regret. Let $V^\star(x_{k,1})$ be the expected value of $V^\star(x_{k,1};H_k)$ i.e. $V^\star(x_{k,1}) = \sum_{\ell} V^\star(x_{k,1}; \ell) \ph(\ell)$. Then the expected regret of a learning algorithm is given as
\begin{equation*}
\reg(\bm{\pi}; \ph(\cdot)) = \sum_{k \in [T]} V^\star(x_{k,1}) - \E_{H_k} \left[ V^{\pi_k}(x_{k,1}; H_k)\right]
\end{equation*}

\subsection{An Equivalent Model of General Discounting}\label{subsec:general-discounting}
We first establish that the problem of minimizing regret in our setting is equivalent to minimizing regret in a different environment, where the goal is to minimize discounted reward over an infinite horizon with a general notion of discounting. By setting $\gamma(h) = \ph(H \ge h)$, the expected reward in \cref{eq:original-undiscounted-reward}   becomes a sum of $T$ expected rewards under the general discounting function $\left\{\gamma(h)\right\}_{h=1}^\infty$.
\begin{align*}
\E\left[\rew(\bm{\pi}; \{H_k\}_{k \in [T]}) \right] =  \sum_{t=1}^T \E\left[\sum_{h=1}^\infty \gamma(h) r(x_{t,h}, a_{t,h}) \lvert x_{k,1} \right]
\end{align*}
Therefore, we consider the equivalent setting where the agent is interacting with the MDP $\calM = (S, \calA, r, \Pm, \bgamma)$ where $\bgamma = \left\{\gamma(h)\right\}_{h=1}^\infty$ is a general discounting factor. We will require the following two properties from the discounting factors:

\begin{enumerate}
\item $\gamma(1) = 1$,
\item $\sum_{h=1}^{\infty} \gamma(h) \le M$ for some universal constant $M>0$.
\end{enumerate}
The first assumption is without loss of generality as we can normalize all the discount factors without affecting the maximization problem. The second assumption guarantees that the optimal policy is well-defined. Note that this assumption rules out hyperbolic discounting $\gamma(h) = \frac{1}{1+h}$,
but does allow discount factors of the form $\gamma(h) = h^{-p}$ for any $p > 1$. Finally, note that our original reformulation of $\gamma(h) = \ph(H \ge h)$ trivially satisfies the first assumption. The second assumption essentially ensures that the horizon length has a finite mean. We will also write $\Gamma(h)$ to define the sum of the tail part of the series starting at $h$ i.e.
\begin{equation}\label{eq:defn-Gamma}
    \Gamma(h) = \sum_{j \ge h} \gamma(j)
\end{equation}

In this new environment, the learner solves the following episodic reinforcement learning problem over $T$ episodes.
\paragraph{Environment: General Discounting}
\begin{enumerate}
\item The starting state $x_{k,1}$ is chosen arbitrarily.
\item The agent maximizes  $\E\left[\sum_{h=1}^{\infty} \gamma(h) r(\xkh,\akh) \lvert x_{k,1} \right]$ over an infinite horizon.
\end{enumerate}
Notice that even though the new environment is episodic, the length of each episode is infinite. So this environment is not realistic, and we are only introducing this hypothetical environment to design our algorithm and analyze its performance. 

Suppose that we are given a learning algorithm $\bm{\pi} = \{\pi_k\}_{k \in [T]}$. We allow the possibility that $\pi_k$ is a non-stationary policy as each $\pi_k$ is used to maximizing a discounted sum of rewards with respect to a general discounting factor and in general the optimal policy need not be stationary. 
A non-stationary policy $\pi_k$ is a collection of policies $\{ \pi_{k,h}\}_{h=1}^{\infty}$ where $\pi_{k,h} : (S \times \calA)^{h-1} \times S \rightarrow \Delta(\calA) $. 
Given a non-stationary policy $\pi_k$ at episode $k$, we define the state-action $Q$ function  and the value function  as
\begin{align*}
Q^{\pi_k}(x,a; \bgamma) &= \E\left[\sum_{h=1}^{\infty} \gamma(h) r(\xkh, \akh)\lvert x_{k,1} = x, a_{k,1}=a \right]\\
V^{\pi_k}(x; \bgamma) &= \E\left[\sum_{h=1}^{\infty} \gamma(h) r(\xkh, \akh)\lvert x_{k,1} = x\right]
\end{align*}
Here $\akh \sim \pi_{k,h}(x_{k,1},a_{k,1},\ldots, x_{k,h-1}, a_{k,h-1},\xkh)$. 
In this environment, we again measure the regret as the sum of sub-optimality gaps over the $T$ episodes.
\begin{definition}
Let the optimal value function be defined as $V^\star(x; \bgamma) = \sup_\pi V^\pi(x; \bgamma)$. Then we define regret for a learning algorithm $\bm{\pi} = \{\pi_k\}_{k \in [T]}$ as
\begin{align}\label{eq:regret-defn}
\reg(\bm{\pi}, \bgamma) = \sum_{k=1}^T V^\star(x_{k,1}; \bgamma) - V^{\pi_k}(x_{k,1}; \bgamma)
\end{align}
\end{definition}
Our next result shows that it is sufficient to minimize regret with respect to the new environment of episodic reinforcement learning. In fact, if any algorithm has regret $\calR(T)$ with respect to the new benchmark, then it has regret at most $\calR(T)$ with respect to the original environment with uncertain episode lengths.

\begin{lemma}\label{lem:why-alternative-regret}
For any learning algorithm $\bm{\pi} = \{\pi_k\}_{k \in [T]}$ we have the following guarantee:
$$
\reg(\bm{\pi}; \ph(\cdot)) \le \reg(\bm{\pi}; \bgamma).
$$
\end{lemma}

We also show that a converse of lemma~\ref{lem:why-alternative-regret} holds with  additional restrictions on the discount factor $\bgamma$.
\begin{lemma}\label{lem:why-alternative-regret-part-2}
Suppose the discount factor $\bgamma$ is non-increasing.
Then there exists a distribution $\ph(\cdot)$ over the episode lengths so that 
$$
\reg(\bm{\pi}; \bgamma) \le \reg(\bm{\pi}; \ph(\cdot)).
$$
\end{lemma}

Because of lemma~\ref{lem:why-alternative-regret}, it is sufficient to bound a learning algorithm's regret for the  environment with infinite horizon and general discounting. Therefore, we now focus on designing a learning algorithm that acts in an episodic setting with uncertain episode lengths, but  analyze its regret in the infinite horizon setting with general discounting.

\section{Algorithm: Regret Minimization under General Discounting}
We now introduce our main algorithm.
Given a non-stationary policy $\pi_k$, we define the state-action function and value function at step $h$ as follows.
\begin{align*}
Q^{\pi_k}_h(x,a) &= \E\left[ \sum_{j=1}^\infty \gamma(j) r(x_{k,h+j-1},a_{k,h+j-1}) \mid   \calH_{h-1}, \xkh = x, \akh = a\right] \\
V^{\pi_k}_h(x) &=  \E\left[ \sum_{j=1}^\infty \gamma(j) r(x_{k,h+j-1},a_{k,h+j-1}) \mid \calH_{h-1}, \xkh = x \right]
\end{align*}
where $\calH_{h-1} = (x_{k,1},a_{k,1},\ldots, a_{k,h-1})$ and $a_{k,h+j} \sim \pi_{k,h+j}(\calH_{h+j-1}, x_{k,h+j})$. Note that, both the state-action $Q$-function and the value function depend on the history $\calH_{h-1}$. Moreover, conditioned on the history, we are evaluating the total discounted reward as if the policy $\{\pi_{k,h+j}\}_{j\ge 0}$ was used from the beginning.  
We first establish some relations regarding the above state-action and value functions. We drop the episode index $k$ for ease of exposition. Given a non-stationary policy $\pi = \{\pi_h\}_{h \ge 1}$ let
\begin{align*}
    &Q^\pi_h(x,a) = r(x,a) + \gamma(2) \cdot  \E\left[\sum_{j=1}^{\infty} \frac{\gamma(j+1)}{\gamma(2)}  r(x_{h+j}, a_{h+j})\lvert \calH_{h-1}, x_h = x, a_h = a \right]\\
    &= r(x,a) + \gamma(2) \E_{x_{h+1} \sim \Pm(\cdot|x,a)}\left[ \E\left[\sum_{j=1}^{\infty} \frac{\gamma(j+1)}{\gamma(2)}   r(x_{h+j+1}, a_{h+j+1})\lvert \calH_{h}, x_{h+1} \right] \right]\\
    &= r(x,a) + \gamma(2) \E_{x_{h+1} \sim \Pm(\cdot|x,a)}\left[ V_{h+1}^\pi (x_{h+1}; \bgamma_2)\right]
\end{align*}
where in the last line we write $\bgamma_2$ to denote the discount factor $\bgamma_2(j) = \frac{\gamma(j+1)}{\gamma(2)}$ and $V_{h+1}^\pi (x_{h+1}; \bgamma_2)$ is the value function at time-step $h$ with respect to the new discount factor $\bgamma_2$. By a similar argument one can write the action-value function with respect to the discount factor $\bgamma_2$ as the following expression.
\begin{align*}
&Q^\pi_h(x,a; \bgamma_2) \\
&= r(x,a) + \gamma_2(2) \E_{x_{h+1} \sim \Pm(\cdot|x,a)}\left[ V_{h+1}^\pi (x_{h+1}; \bgamma_2)\right] \\
&= r(x,a) + \frac{\gamma(3)}{\gamma(2)} \E_{x_{h+1} \sim \Pm(\cdot|x,a)}\left[ V_{h+1}^\pi (x_{h+1}; \bgamma_3)\right] 
\end{align*}
where the discount factor $\bgamma_3$ is given as $\bgamma_3(j) = \frac{\gamma(j+2)}{\gamma(3)}$. In general, we have the following  relation.
\begin{align}
&Q^\pi_h(x,a; \bgamma_k) = r(x,a)  + \frac{\gamma(k+1)}{\gamma(k)} \E_{x_{h+1} \sim \Pm(\cdot|x,a)}\left[ V_{h+1}^\pi (x_{h+1}; \bgamma_{k+1})\right] \label{eq:general-q-value-recursion}
\end{align}
where the discount factor $\gamma_k$ is defined as $\gamma_k(j) = \frac{\gamma(j+k-1)}{\gamma(k)}$ for $j=1,2,\ldots$.
Notice that when $\bgamma$ is a geometric discounting, we only need  equation.
\begin{align}\label{eq:geometric-q-value}
Q^\pi_h(x,a) = r(x,a) + \gamma \E_{x_{h+1} \sim \Pm(\cdot|x,a)}\left[ V_{h+1}^\pi (x_{h+1}) \right] 
\end{align}

\begin{algorithm}[!t]
\DontPrintSemicolon
\KwInput{Discount factor $\{\gamma(h)\}_{h=1}^{\infty}$, parameter $\Delta$}
$\mathcal{H} \leftarrow \emptyset$.\\
\For{$h = 1,\ldots, N(\Delta)$}
{
	Set $Q_{1,h}(x,a) \leftarrow \sum_{j=1}^{\infty} \gamma_h(j) = \frac{1}{\gamma(h)} \sum_{j=1}^\infty \gamma(j+h-1)$ for all $x \in S$ and $a \in \calA$.
}
\For{$t = 1,\ldots, T$}
{
	Update-Q-values($\mathcal{H}, \bgamma, \Delta$).\\
	Receive state $x_{t,1}$.\\
	\For{$h=1,\ldots$}
	{
	    \If{$h \le N(\Delta)$}
	    {
		    Take action $a_{t,h} = \argmax_a Q_{t,h}(x_{t,h}, a)$\\
		    Update $\mathcal{H} = \mathcal{H} \cup (x_{t,h}, a_{t,h}, x_{t,h+1})$\\
		    \If{$x_{t,h+1}$ is a terminal state}
		    {
			    Continue to the next episode.\\
		    }
		}
		{
		    Take an arbitrary action.\\
		}
	}
}
\caption{UCB-VI Generalized\label{alg:vibi}}
\end{algorithm}

\paragraph{Description of the Learning Algorithm}: The sequence of recurrence relations~\cref{eq:general-q-value-recursion} motivates our main algorithm (\ref{alg:vibi}). Our algorithm is based on the upper confidence value iteration algorithm (\texttt{UCBVI}~\cite{AOM17}). In an episodic reinforcement learning setting with fixed horizon length $H$, \texttt{UCBVI} uses backward induction to update the $Q$-values at the end of each episode, and takes greedy action according to the $Q$-table. 

However, in our setting, there is no fixed horizon length and the $Q$-values are related through an infinite sequence of recurrence relations. So, algorithm~\ref{alg:vibi} considers a truncated version of the sequence of recurrence relations~\cref{eq:general-q-value-recursion}. In particular, given an input discount factor $\{\gamma(h)\}_{h=1}^\infty$\footnote{Recall that $\gamma(h) = \ph(H \ge h)$.} and a parameter $\Delta$, algorithm~\ref{alg:vibi} first determines $N(\Delta)$ as a measure of effective length of the horizon. In particular, we set $N(\Delta)$ to be an index so that $\Gamma(N(\Delta)) = \sum_{j\ge N(\Delta)} \gamma(j) \le \Delta$. Note that, such an index $N(\Delta)$ always exists as we assumed that the total sum of the discounting factors converges. Then algorithm~\ref{alg:vibi} maintains an estimate of the $Q$ value for all possible discount factors up to $N(\Delta)$ i.e. $\bgamma_k$ for $k=1,\ldots, N(\Delta)$.

The details of the update procedure is provided in the appendix. In the update procedure, we first set the $(N(\Delta) + 1)$-th $Q$-value to be $\Delta / \gamma(N(\Delta) + 1)$ which is always an upper bound on the $Q$-value with discount factor $\gamma_{N(\Delta) + 1}$ because of the way algorithm~\ref{alg:vibi} sets the value $N(\Delta)$. Then starting from level $N(\Delta)$, we update the $Q$-values through backward induction and \cref{eq:general-q-value-recursion}.

Note that our algorithm needs to maintain $N(\Delta)$ action-value tables. We will later show that in order to obtain sub-linear regret we need to choose $\Delta$ based on the particular discount factor. In particular, for the geometric discount factor $\gamma(h) = \gamma^{h-1}$ we need to choose  $N(\Delta) = \frac{\log T}{\log(1/\gamma)}$. On the other hand, discounting factor of the form $\gamma(h) = 1/h^p$ requires $N(\Delta) = O\left(T^{1/(2p-1)}\right)$.

\section{Analysis}
The next theorem provides an upper bound on the regret $\reg(\bm{\pi}; \bgamma)$. In order to state the theorem, we need a new notation. Let the function $t: N \rightarrow \R$ be defined as 
$$
t(h) = \left\{ \begin{array}{cc}
1 & \textrm{ if } h = 1\\
\frac{\gamma(h)}{\gamma(1)} \prod_{j=2}^h \left(1 + \frac{\gamma(j)}{{j^{\beta}}\Gamma(j)} \right)& \textrm{ o.w. } 
\end{array}
\right.
$$

Note that the function $t$ is parameterized by the parameter $\beta$ and depends on the discount factor $\gamma(\cdot)$.
\begin{theorem}[Informal]\label{thm:main-regret-bound}
With probability at least $1-\delta$, Algorithm~\ref{alg:vibi} has the following regret.
\begin{align*}
    \reg(\bm{\pi}; \bm{\gamma}) &\le \frac{\Delta T}{\gamma(N(\Delta) + 1)} t(N(\Delta) + 1)  +  \max_{h \in [N(\Delta)]} t(h) \frac{\Gamma(h+1)}{\gamma(h)} \tilde{O}\left(\sqrt{SATN(\Delta)} \right)
\end{align*}
\end{theorem}

Theorem~\ref{thm:main-regret-bound} states a generic bound that holds for any discount factor. The main terms in the bound are $O\left( \sqrt{SATN(\Delta)}\right)$, $\Delta T$, and several factors dependent on the discount factor $\gamma$. We now instantiate the bound for different discount factors by choosing appropriate value of $\Delta$ and the parameter $\beta$.

\begin{corollary}\label{cor:poly-decay}
Consider the discount factor $\gamma(h) = h^{-p}$. 
For $p \ge 2$ and $T \ge O(S^3 A)$ we have
$$
\textstyle \reg(T) \le \tilde{O}\left( S^{1/2} A^{1/2} T^{\frac{1}{2-1/p}} \right)
$$
and for $1 < p < 2$ and $T \ge O\left( (S^3 A)^{\frac{2p-1}{p-1}}\right)$ we have
$$
\textstyle \reg(T) \le \tilde{O}\left( (p-1)^{-\frac{p}{p-1}} S^{1/2} A^{1/2} T^{\frac{1}{2-1/p}}\right)
$$
\end{corollary}

We prove corollary~\ref{cor:poly-decay} by substituting $\beta = p-1$ and $\Delta = O\left(T^{-\frac{p-1}{2p-1}}\right)$. Note that this result suggests that as $p$ increases to infinity, the regret bound converges to $O(\sqrt{T})$. This also suggests that for exponentially decaying discounting factor, our algorithm should have exactly $O(\sqrt{T})$ regret. We verify this claim next.

\begin{corollary}\label{cor:exp-decay}
Consider the discount factor $\gamma(h) = \gamma^{h-1}$ for $\gamma \in [0,1)$ and suppose $T \ge \frac{S^3 A}{(1-\gamma)^{4}}$. Then algorithm \ref{alg:vibi} has regret at most
$$
\reg(T) \le \tilde{O}\left({\sqrt{SAT}}/{(1-\gamma)^{1.5}}\right)
$$
\end{corollary}
Here we substitute $\beta = 3/2$ and $\Delta = T^{-1}/(1-\gamma)$.
Our regret bound for the geometric discounting matches the minimax optimal regret bound of the non-episodic setting of \cite{HZG21}. 

\paragraph{Proof Sketch of Theorem~\ref{thm:main-regret-bound}}: We now give an overview of the main steps of the proof. Although the proof is based upon the proof of the UCB-VI algorithm~\cite{AOM17}, there are several differences.

\begin{itemize}
    \item 
    Let $V^\star_h(\cdot)$ be the optimal value function under discounting factor $\gamma_h(\cdot)$ i.e. $V^\star_h(x) = \sup_\pi V^\pi(x;\gamma_h)$. We first show that the estimates $V_{k,h}$ maintained by Algorithm~\ref{alg:vibi} upper bound the optimal value functions i.e. 
    $
    V_{k,h}(x) \ge V^\star_{h}(x)
    $  for any $k,h \in [N(\Delta)]$.
    \item Let $\tilde{\Delta}_{k,h} = V_{k,h} - V^{\pi_k}_h$. Then regret can be bounded as
    \begin{align*}
& \reg(\bm{\pi}; \bgamma) = \sum_{k=1}^T V^\star(x_{k,1}) - V^{\pi_k}_1(x_{k,1})
\le \sum_{k=1}^T V_{k,1}(x_{k,1}) - V^{\pi_k}_1(x_{k,1}) \le \sum_{k=1}^T \tilde{\Delta}_{k,1}(x_{k,1})
\end{align*}
\item Let $\tdkh = \tilde{\Delta}_{k,h}(\xkh)$. Then, the main part of the proof of theorem~\ref{thm:main-regret-bound} is establishing the following recurrent relation.
\begin{align*}
& \tdkh \le \frac{\gamma(h+1)}{\gamma(h)} \left(1 + \frac{\gamma(h+1)}{(h+1)^{\beta} \Gamma(h+1)} \right) \tilde{\delta}_{k,h+1} + \sqrt{2L}\bar{\eps}_{k,h} + e_{k,h} + b_{k,h}+\eps_{k,h}+f_{k,h}
\end{align*}
Here $\bar{\eps}_{k,h}$ and $\eps_{k,h}$ are Martingale difference sequences and $b_{k,h}, e_{k,h}, f_{k,h}$ are either the bonus term or behave similarly as the bonus term. 
\item We complete the proof by summing the recurrence relation above over all the episodes and from $h=1$ to $N(\Delta)$. Although \cite{AOM17} established a similar recurrence relation, there are two major differences. First the multiplicative factor in front of $\tilde{\delta}_{k,h+1}$ is changing with time-step $h$ and is not a constant. This is because the backward induction step uses \cref{eq:general-q-value-recursion} in our setting. Second, after expanding the recurrence relation from $h=1$ to $N(\Delta)$ the final term is no longer zero and an extra $O(\Delta T)$ term shows up in the regret bound.
\end{itemize}


\begin{algorithm}[!t]
\KwInput{Horizon Length $H^\star = N(\Delta)$.}
\DontPrintSemicolon
Set block length $B = \sqrt{T}\log T \log(\log(T)/\delta)$.\\
Set $\hat{\gamma}_0$ to be an arbitrary discount factor.\\
\For{$j=0,1,\ldots, \log(T/B) - 1$}
{
    \If{$j > 0$}
    {   $\hat{\gamma}_j(h) = 1-\hat{F}_H(h-1)$ forall $h$.}
    $\hat{\Delta}_j = \sum_{h \ge H^\star + 1} \hat{\gamma}_j(h)$.\\
    Run algorithm~\ref{alg:vibi} for $2^j B$ episodes with inputs $\hat{\gamma}_j$ and $\hat{\Delta}_j$.\\
    \tcc{update empirical distribution function}
    $\hat{F}_H(h) = \frac{1}{2^j B} \sum_{t=0}^{2^j B} \one\set{H_t \le h}$.\\
}
\caption{Estimating Unknown Discount Factor\label{alg:estimate-gamma}}
\end{algorithm}

\section{Estimating the Discount Function}
In this section we consider the situation when the discount function $\gamma(h) = \ph(H \ge h)$ is not unknown. We start with the assumption that the optimal value of $N(\Delta)$ (say $H^\star$) is known. The next lemma bounds the regret achieved by running an algorithm with $N(\Delta) = H^\star$ with the true discounting $\gamma$ and an estimate of the discounting $\hat{\gamma}$. Our algorithm partitions the entire sequence of $T$ episodes into blocks of lengths $B, 2B, 2^2 B,\ldots, 2^s B$ for $s = \log(T/B) - 1$. At the end of each block the algorithm recomputes an estimate of $\gamma$. Recall that we defined $\gamma(h) = \Pr(H \ge h)$. Since every episode we get one sample from the distribution of $H$ (the random length of the current episode) we can use the empirical distribution function of horizon length to obtain $\hat{\gamma}$. At the end of block $B$, the algorithm computes $\hat{\gamma}_B$, and runs algorithm \ref{alg:vibi} with this estimate and $\hat{\Delta}_B = \hat{\Gamma}_B(H^\star + 1) = \sum_{h \ge H^\star + 1} \hat{\gamma}_B(h)$ for the block $B+1$.

\begin{theorem}[Informal]\label{thm:bound-alg-est}
When run with  horizon length $H^\star$, algorithm~\ref{alg:estimate-gamma} has the following regret bound with probability at least $1-\delta$
\begin{align*}
    &\reg(\pi;\bgamma) \le \min_{L \in [T]}\left(T \Gamma(L+1) + 2L \log(T) \sqrt{T} )\right) + \Gamma(H^\star) T\\
    & + \max_{h \in[H^\star]} \frac{t(h)}{\gamma(h)}g(h) \Gamma(h+1)  \frac{O(T^{-1/4})}{\Gamma(h+1)} \tilde{O}\left(\sqrt{SATH^\star}\right)
\end{align*}
where $g(h) = \exp\left\{O\left(\sum_{k=2}^h \frac{T^{-1/4}}{\gamma(k) + k^{\beta} \Gamma(k)} \right) \right\}$.
\end{theorem}

\begin{figure*}[!t]
\centering
\begin{subfigure}[b]{0.3\textwidth}
\centering
  \includegraphics[width=\linewidth]{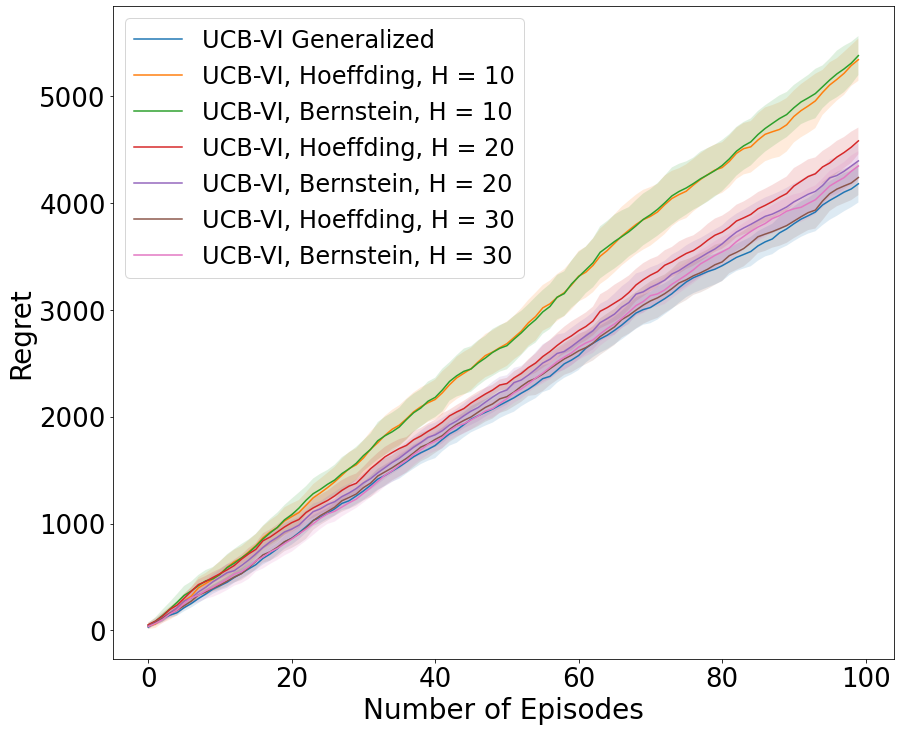}
  \caption{Geometric, $\gamma=0.9$}\label{fig:geom-p-09}
\end{subfigure}\hfill
\begin{subfigure}[b]{0.3\textwidth}
\centering
  \includegraphics[width=\linewidth]{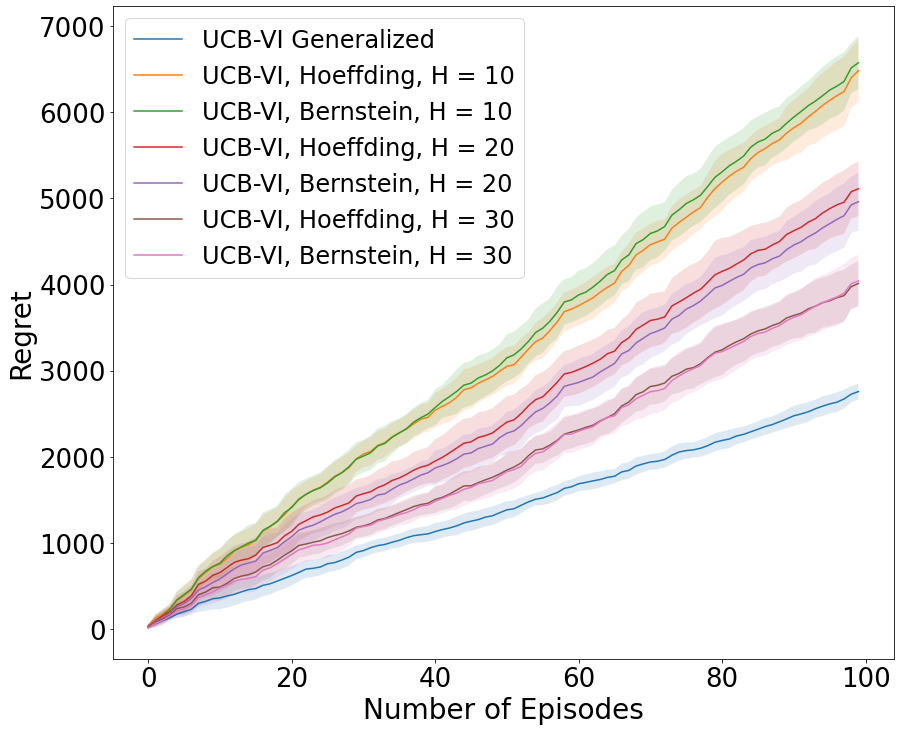}
  \caption{Geometric, $\gamma=0.95$}\label{fig:geom-p-095}
\end{subfigure}\hfill
\begin{subfigure}[b]{0.3\textwidth}%
\centering
  \includegraphics[width=\linewidth]{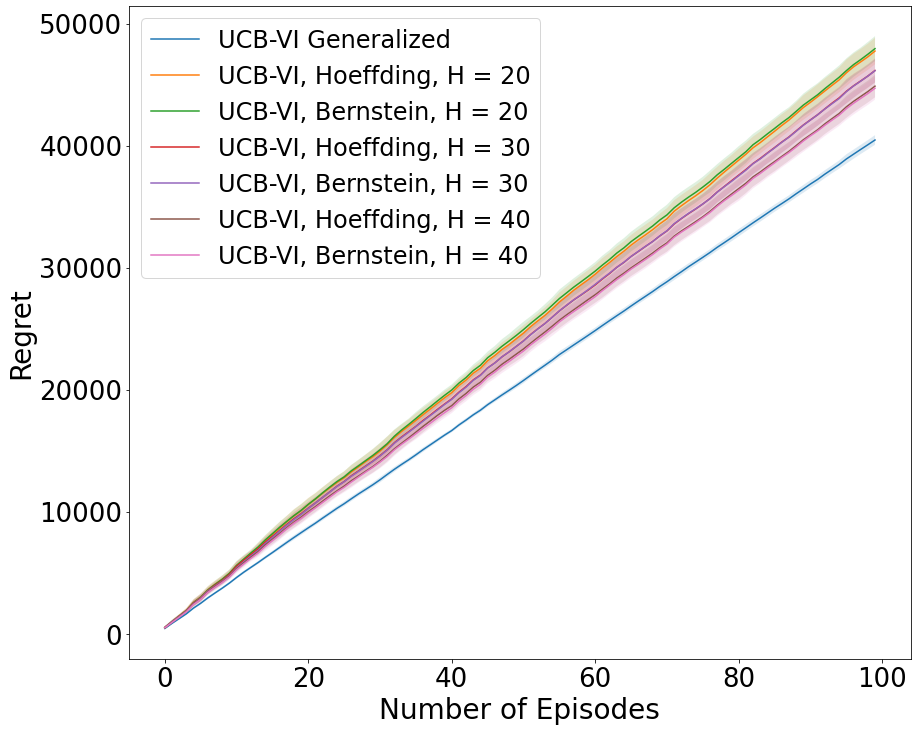}
  \caption{Geometric, $\gamma=0.975$}\label{fig:geom-p-099}
\end{subfigure}
\hfill
\begin{subfigure}[b]{0.3\textwidth}
\centering
  \includegraphics[width=\linewidth]{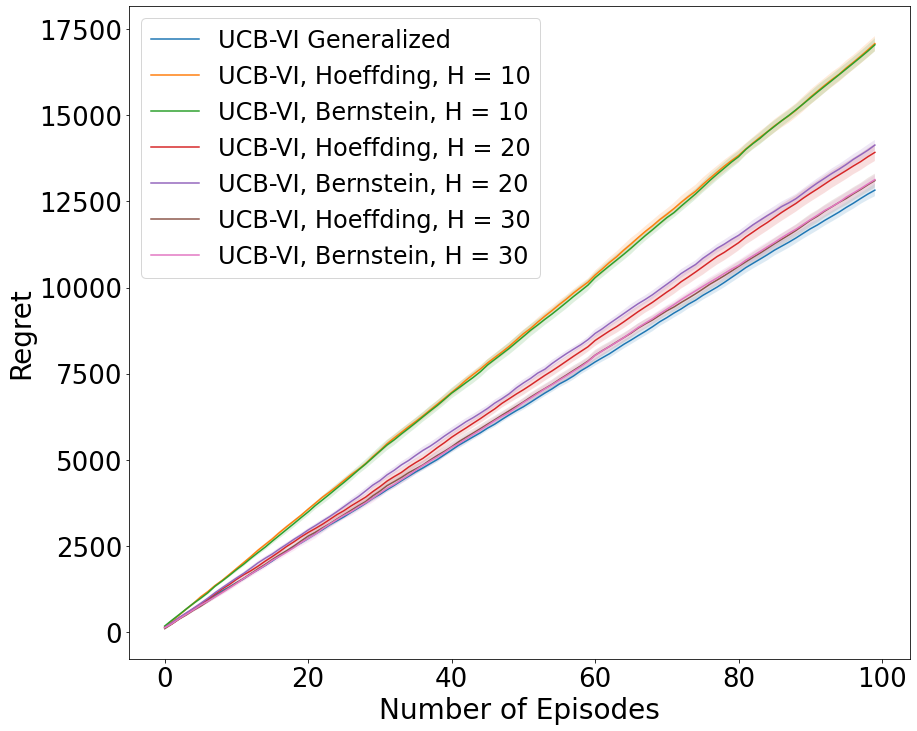}
  \caption{Polynomial, $p=1.4$}\label{fig:poly-1p4}
\end{subfigure}\hfill
\begin{subfigure}[b]{0.3\textwidth}
\centering
  \includegraphics[width=\linewidth]{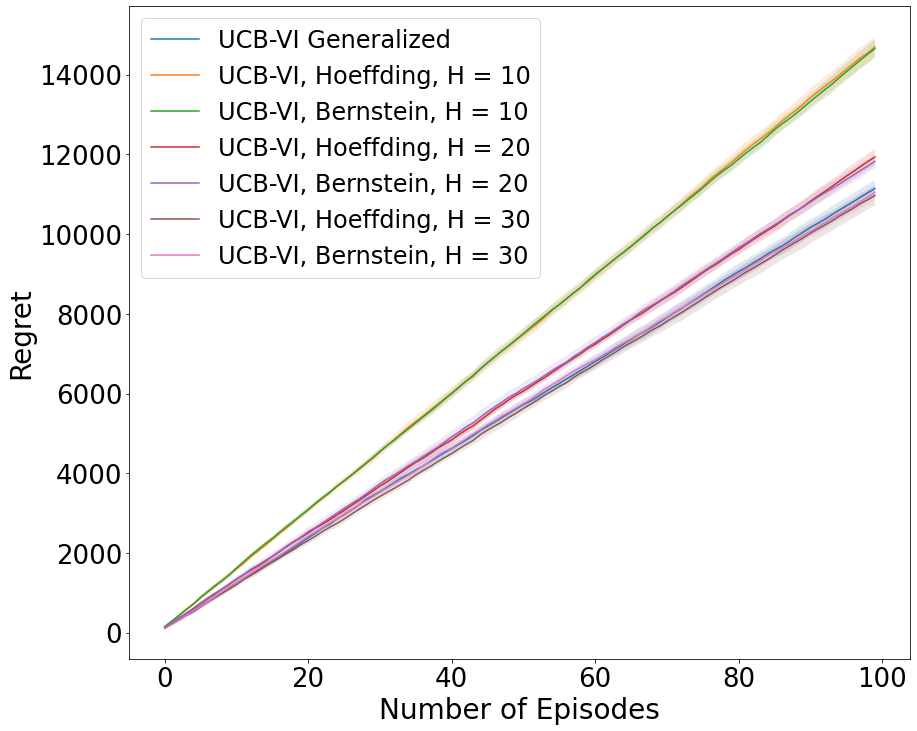}
  \caption{Polynomial, $p=1.6$}\label{fig:poly-1p6}
\end{subfigure}\hfill
\begin{subfigure}[b]{0.3\textwidth}%
\centering
  \includegraphics[width=\linewidth]{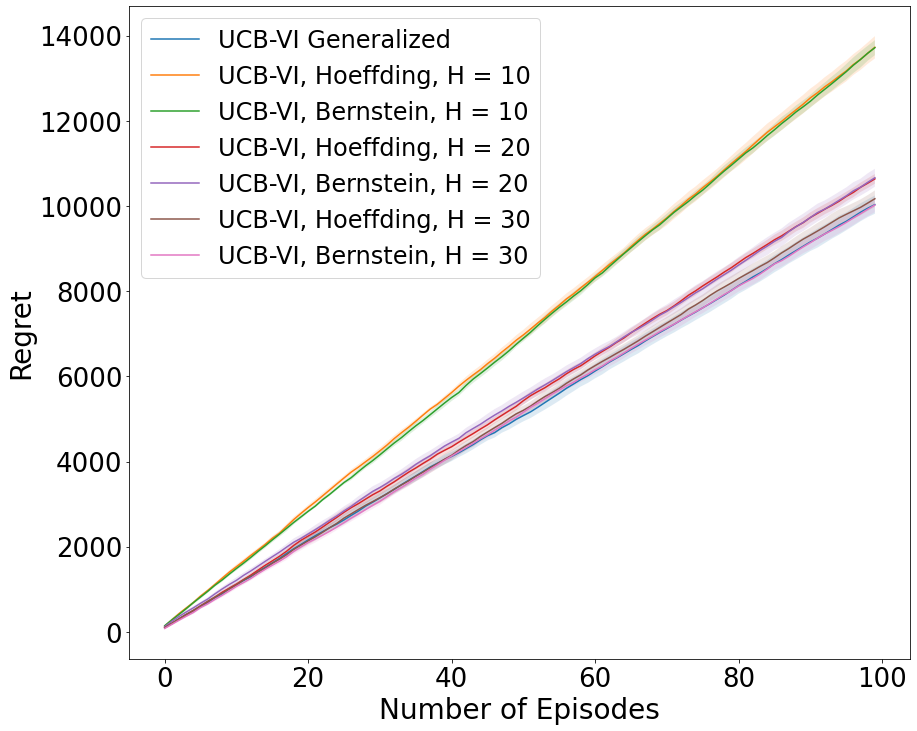}
  \caption{Polynomial, $p=2.0$}\label{fig:poly-2}
\end{subfigure}
\hfill
\begin{subfigure}[b]{0.3\textwidth}
\centering
  \includegraphics[width=\linewidth]{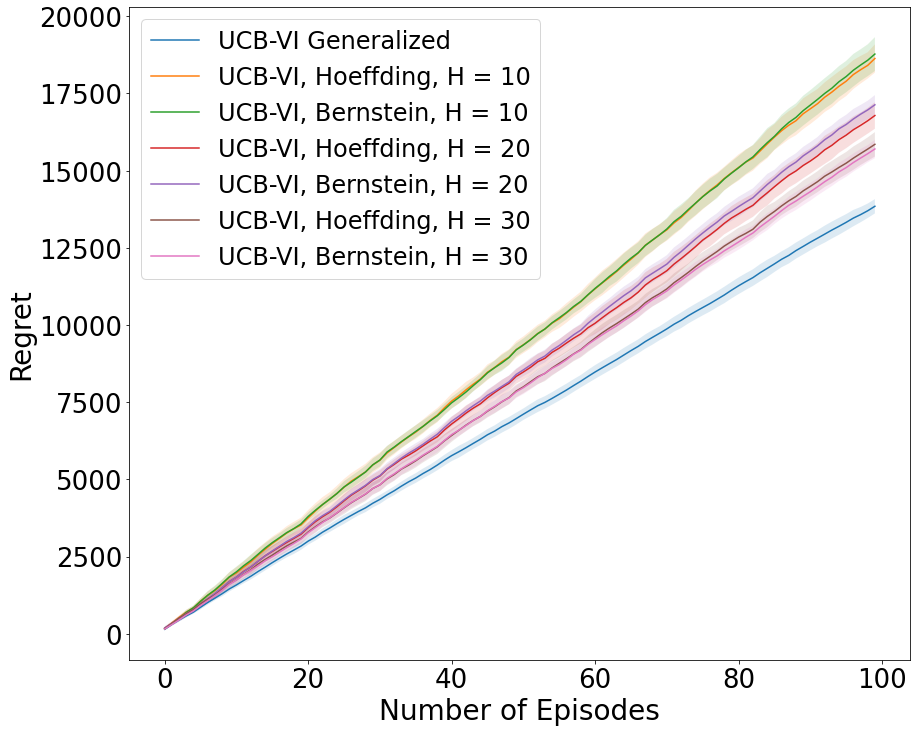}
  \caption{Quasi-Hyperbolic, $\beta=0.7$}\label{fig:qh-0p7}
\end{subfigure}\hfill
\begin{subfigure}[b]{0.3\textwidth}
\centering
  \includegraphics[width=\linewidth]{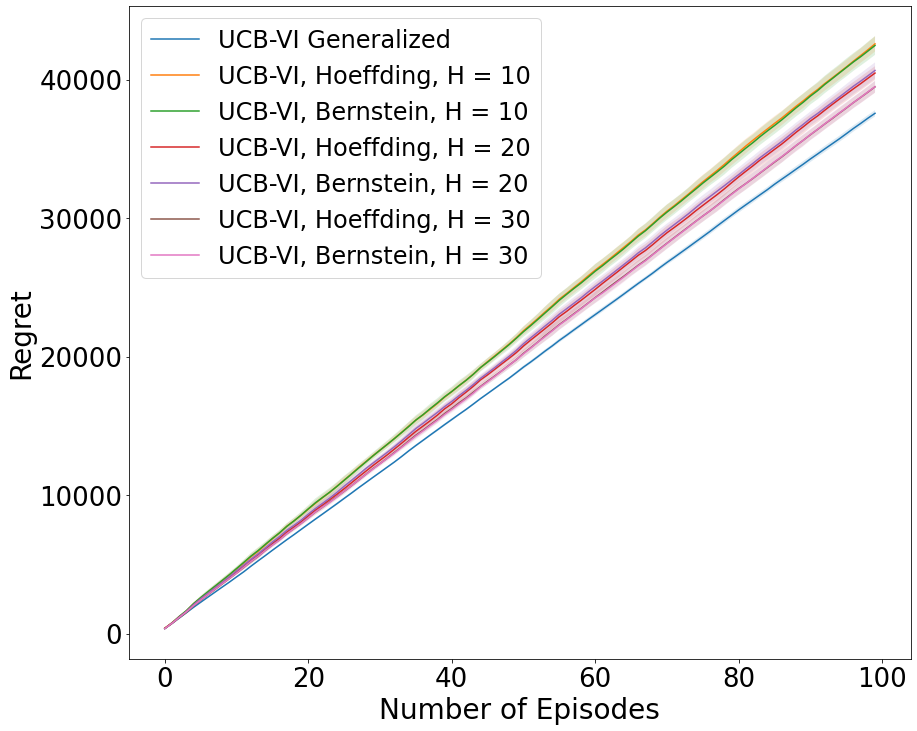}
  \caption{Quasi-Hyperbolic, $\beta=0.8$}\label{fig:qh-0p8}
\end{subfigure}\hfill
\begin{subfigure}[b]{0.3\textwidth}%
\centering
  \includegraphics[width=\linewidth]{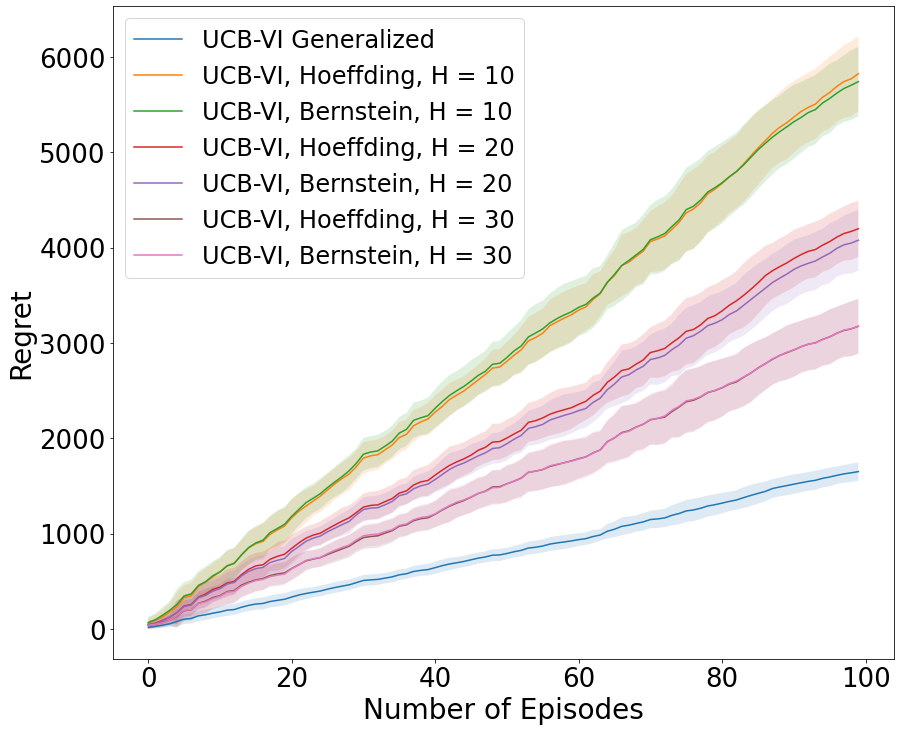}
  \caption{Quasi-Hyperbolic, $\beta=0.9$}\label{fig:qh-09}
\end{subfigure}
\caption{Comparison of our algorithm with different variants of UCB-VI on the Taxi environment~\cite{Dietterich00}. The regret is measured over $100$ episodes, and the length of each episode is drawn independently from a given distribution. Each plot shows average regret and standard error from $10$ trials. \label{fig:taxiv3} }
\end{figure*}

Our proof relies on bounding the estimation error of $\hat{\gamma}$ and $\hat{\Gamma}$. We can use the classical DKW inequality~\cite{DKW56} to bound the maximum deviation between empirical CDF ($\widehat{\mathtt{P}}_{\mathtt{H}}(\cdot)$) and true CDF ($\ph$). Through a union bound over the $\log(T)$ blocks, this immediately provides a bound between $\norm{\hat{\gamma}_j - \gamma_j}_\infty$ for all $j \in [\log(T/B)]$. However, we also need to  bound the distance between $\hat{\Gamma}_j(\cdot)$ and $\Gamma(\cdot)$ for all $j$ (defined in \eqref{eq:defn-Gamma}). A naive application of DKW inequality results in an additive bound between $\hat{\Gamma}_j(h)$ and $\Gamma(h)$ that grows at a rate of $h$. This is insufficient for our case to get a sublinear regret bound. 
However, we show that we can use the multiclass fundamental theorem~\cite{SSS14} to derive an error bound that grows at a rate of $\sqrt{\log h}$ and this is sufficient for our proof.

The main challenge in the proof of theorem~\ref{thm:bound-alg-est} is controlling the growth of the term $t(h) / \gamma(h)$. Notice that this term is a product of $h$ terms of the form $1 + \frac{\gamma(k)}{k^\beta \Gamma(k)}$, so any error in estimating $\gamma$ could blow up the product by a factor of $h$. We could show that the regret is multiplied by an additional function $g(h)$ which is parameterized by $\beta$.
We next instantiate theorem~\ref{thm:bound-alg-est} for different discount factors and show that we can obtain regret bounds similar to corollary~\ref{cor:poly-decay}, and \ref{cor:exp-decay}  up to logarithmic factors.

\begin{corollary}\label{cor:poly-decay-est}
Consider the discount factor $\gamma(h) = h^{-p}$ for $p \ge 2$. Then the regret of algorithm~\ref{alg:estimate-gamma} is 
$$
\reg(T) \le \left\{ \begin{array}{cc}
    \tilde{O}\left(\sqrt{SA} T^{\frac{p+1}{2p}} \right) & \textrm{ if } T \ge O\left( (S^{3/2} A^{1/2})^p\right) \\
    \tilde{O}\left(S^2 A T^{\frac{1}{2-1/p}} \right) &  \textrm{ if } T \le O\left( (S^{3/2} A^{1/2})^p\right) 
\end{array}\right.
$$
\end{corollary}

\begin{corollary}\label{cor:geom-decay-est}
Consider the geometric discount factor $\gamma(h) = \gamma^{h-1}$ for $\gamma \in [0,1)$ and suppose $\frac{T}{\log^3 T} \ge \frac{S^3 A}{(1-\gamma)^4}$. Then algorithm \ref{alg:estimate-gamma} has regret at most
$
\tilde{O}\left({\sqrt{SAT}}/{(1-\gamma)^{1.5}} \right).
$
\end{corollary}

For the polynomial discounting we get a regret of the order of $T^{(p+1)/2p}$ which is worse than the regret bound of theorem~\ref{thm:main-regret-bound} by a factor of $T^{1/2p}$. However, the difference goes to zero as $p$ increases and approaches the same limit of $\tilde{O}(\sqrt{T})$. On the other hand, for geometric discounting we recover the same regret as corollary \ref{cor:exp-decay}. Interestingly, \citet{HZG21} obtained a similar bound on regret for the non-episodic setting where the learner maximizes her long-term geometrically distributed reward.

\textbf{Unknown $N(\Delta)$}: Note that algorithm~\ref{alg:estimate-gamma} takes as input the optimal value of $N(\Delta)$ or $H^\star$. However, this problem can be handled through a direct application of model selection algorithms in online learning~\cite{CDDG+21}. Let $\reg(H^\star)$ be the regret when algorithm~\ref{alg:estimate-gamma} is run with true $H^\star$. We now instantiate algorithm~\ref{alg:estimate-gamma} for different choices of $H^\star$ and perform model selection over them. In particular, we can consider $H^\star = 2, 2^2, \ldots, 2^{O(\log T)}$ as it is sufficient to consider $H^\star = O(T)$. Moreover, given true $H^\star$ there exists $\tilde{H} \le 2H^\star$ for which the regret is increased by at most a constant. This step requires bounding $\frac{t(H^\star)}{\gamma(H^\star)}/ \frac{t(\tilde{H)}}{\gamma(\tilde{H})}$ and is constant for the discounting factors considered in the paper. We now apply algorithm~1 from \cite{CDDG+21} to the collection of $O(\log T)$ models and obtain a regret bound of at most $O\left(\sqrt{\log T} \reg(\tilde{H})\right) = \tilde{O}(\reg(H^\star))$.



\section{Experiments}
We evaluated the performance of our algorithm on the Taxi environment, a $5\times 5$ grid-world environment introduced by \cite{Dietterich00}. The details of this environment is provided in the appendix, since the exact details are not too important for understanding the experimental results.
We considered $100$ episodes and each episode  length  was generated uniformly at random from the following distributions. \footnote{Here $\gamma(h)$ refers to probability that the episode lengths exceeds $h$ i.e. $\gamma(h) = \Pr(H \ge h)$.}
\begin{enumerate}
    \item Geometric discounting $\gamma(h) = \gamma^{h-1}$.
    \item Polynomial discounting $\gamma(h) = h^{-p}$.
    \item Quasi-Hyperbolic discounting $\gamma(h) = \beta^{\one \set{h > 1}} \gamma^{h-1}$
\end{enumerate}
Figure~\ref{fig:taxiv3} shows some representative parameters for three different types of discounting. For the geometric discounting, we show $\gamma = 0.9, 0.95$ and $0.975$. For the polynomial discounting we generated the horizon lengths from a polynomial with $p \in \set{1.4, 1.6, 2.0}$ and  added an offset of $20$. Finally, for the Quasi-hyperbolic discounting, we fixed $\gamma$ at $0.95$ and considered three  values of $\beta$: $0.7, 0.8$, and $0.9$.

We compared our algorithm~\eqref{alg:vibi} with two variants 
of UCB-VI~\cite{AOM17} -- (a) UCB-VI-Hoeffding computes bonus terms using Chernoff-Hoeffding inequality, and (b) UCB-VI-Bernstein computes bonus terms using Bernstein-Freedman inequality. It is known that when the horizon length is fixed and known, UCB-VI-Bernstein achieves minimax optimal regret bounds. We implemented two versions of UCB-VI with three different assumed horizon lengths.

Figure~\ref{fig:taxiv3} shows that, for several situations, our algorithm  strict improves in regret compared to all the other variants of UCB-VI. These include Geometric discounting ($\gamma = 0.95$ and $0.975$) and Quasi-Hyperbolic discounting (all possible choices of $\beta$). For the other scenarios (e.g. polynomial discounting), our algorithm performs as well as the best version UCB-VI.
Figure~\ref{fig:taxiv3} also highlights the importance of choosing not only the  right horizon length but also the correct update equation in backward induction. Consider for example, figure~\ref{fig:geom-p-095} for the geometric discounting with $\gamma = 0.95$. Here the expected horizon length is $\frac{1}{1-\gamma} = 20$. However, different UCB-VI variants (horizon lengths $10,20,30$ and Bernstein and Hoeffding variants) perform worse. Our algorithm benefits by choosing the right effective horizon length, and also the correct update equation~\eqref{eq:geometric-q-value}.
\section{Conclusion}\label{sec.conclusion}
In this paper, we have designed reinforcement learning algorithms when the episode lengths are uncertain and drawn from a fixed distribution. Our general learning algorithm (\ref{alg:vibi}) and result (\cref{thm:main-regret-bound}) can be instantiated for different types of distributions to obtain sub-linear regret bounds. 
Some interesting directions of future work include extension of our algorithm to function approximation~\cite{JYWJ20}, changing probability transition function~\cite{JZBJ18}, etc. We are also interested in other models of episode lengths. For example, one can consider a setting where the lengths are adversarially generated but there is a limit on the total amount of change. This is similar to the notion of variation budget~\cite{BGZ14} considered in the literature on non-stationary multi-armed bandits.

\section*{Acknowledgements}
Goran Radanovi\'{c} acknowledges that his research was, in part, funded by the Deutsche Forschungsgemeinschaft (DFG, German Research Foundation) – project number 467367360.

\printbibliography


\clearpage
\appendix 
\onecolumn

\section{Pseudocode of Update-Q-Values}
\begin{algorithm}[!h]
\DontPrintSemicolon
\KwInput{Discount factor $\{\gamma(h)\}_{h=1}^{\infty}$, parameter $\Delta$, episode $k$, and dataset $\mathcal{H}$.}
$N_k(x,a,y) = \sum_{(x',a',y') \in \mathcal{H}} \one\set{x'=x, a'=a, y'=y}$ for all $(x,a,y) \in \calS \times \calA \times \calS$.\\
$N_k(x,a) = \sum_y N_k(x,a,y)$ for all $(x,a) \in \calS \times \calA$.\\
Set $\hat{P}(y|x,a) = \frac{N_k(x,a,y)}{N_k(x,a)}$ for all $(x,a)$ s.t. $N_k(x,a) > 0$. Otherwise, $\hat{P}(y|x,a) = \frac{1}{S}$\\

\tcc{Update all the $Q$ values}
$Q_{k,N(\Delta) + 1}(x,a) \leftarrow \frac{\Delta}{\gamma\left(N(\Delta) +  1\right)}$ for all $x \in S$ and $a \in \calA$.\\
$V_{k, N(\Delta) + 1}(x,a) \leftarrow \frac{\Delta}{\gamma\left( N(\Delta) + 1\right)}$ for all $x \in S$.\\
\For{$h=N(\Delta),\ldots,1$}
{
\tcc{Define $[\hat{P}_k V_{k,h+1}](s,a) = \E_{x\sim \hat{P}_k(\cdot|s,a)} \left[ V_{k,h+1}(x)\right]$ and let $\Gamma(h) = \sum_{j \ge h} \gamma(j)$.}
\begin{align*}
&\ucb_{k,h}(s,a) = \frac{3 \Gamma(h+1)}{\gamma(h)} \frac{\ln(SATN(\Delta) / \delta)}{\sqrt{N_k(s,a)}} \\
&Q_{k,h}(s,a) = \min\left\{Q_{k-1,h}(s,a), r(s,a) \right. \\
&+ \left. \frac{\gamma(h+1)}{\gamma(h)}[\hat{P}_k V_{k,h+1}](s,a) 
+  \ucb_{k,h}(s,a)\right\}\\
&V_{k,h}(s,a) = \max_{a \in \calA} Q_{k,h}(s,a)
\end{align*}

}
\caption{Update-Q-values\label{alg:update-q}}
\end{algorithm}

\section{Additional Experimental Details}
In the Taxi environment~\cite{Dietterich00} environment, there are four randomly chosen designated locations indicated by Red, Green, Blue, and Yellow. The driver (learner) starts at a random square and the passenger is at a random location. The taxi
    drives to the passenger's location, picks up the passenger (from one of the four locations), drives to the
    passenger's destination (another one of the four specified locations), and
    then drops off the passenger. There are six actions and the observations include the location of the taxi and the starting point and the destination of the passenger. Each step incurs a cost of $1$. There is a reward of $20$ for successfully delivering a passenger, and a reward of $-10$ for executing pickup and drop-off actions illegally. We used the Open-AI implementation \texttt{Taxi-v3}~\cite{openai-gym} to interact with the environment.
    
    All experiments were conducted on a computer cluster with machines equipped with2Intel XeonE5-2667 v2 CPUs with 3.3GHz (16 cores) and 50 GB RAM.
\section{Equivalence between $\reg(\bm{\pi}; \bgamma)$ and $\reg(\bm{\pi}; \ph(\cdot))$ }
\subsection{Proof of Lemma~\ref{lem:why-alternative-regret}}

\begin{proof}
Consider any learning algorithm $\bm{\pi} = \{ \pi_k\}_{k \in [T]}$. 
\begin{align*}
\E_{H_k}\left[V^{\pi_k}(x_{k,1}; H_k)\right] &= \sum_{\ell} \ph(\ell) \E\left[V^{\pi_k}(x_{k,1}; \ell) \right] = \E\left[\sum_{\ell} \ph(\ell) \sum_{h=1}^\ell r(\xkh,\akh) \right]\\
&= \E\left[\sum_{h=1}^\infty \ph(H_k \ge h) r(\xkh,\akh) \right] = \E\left[\sum_{h=1}^\infty \gamma(h) r(\xkh,\akh) \right]\\
&= V^{\pi_k}(x_{k,1};\bgamma)
\end{align*}
Therefore, we have the following result.
\begin{equation}\label{eq:reward-equivalence-1}
\sum_{k=1}^T \E_{H_k}\left[V^{\pi_k}(x_{k,1}; H_k)\right] = \sum_{k=1}^T V^{\pi_k}(x_{k,1}; \bgamma)
\end{equation}
 We now upper bound the quantity $\sum_\ell V^\star(x_{k,1};\ell) \ph(\ell)$. Let $\pi^\star_{k,\ell}$ be the policy that maximizes the $\ell$-step value function i.e. $V^\star(x_{k,1};\ell) = V^{\pi^\star_{k,\ell}}(x_{k,1}; \ell)$. Let us also define a new policy $\pi_\infty$ as follows.
\begin{enumerate}
\item At step $h$, draw an index $j \in [h,\infty)$ according to the distribution $\frac{\ph(\cdot)}{\ph(H \ge h)}$.
\item Take action $\pi^\star_{k,j}(x_{k,1}, a_{k,1},\ldots, x_{k,h-1},a_{k,h-1},\xkh)$.
\end{enumerate}
Then the infinite-horizon discounted sum of rewards under the policy $\pi_\infty$ is given as
\begin{align*}
\E_{\pi_\infty}\left[\sum_{h=1}^\infty \gamma(h)r(\xkh,\akh) \right] &= \E\left[\sum_{h=1}^\infty \ph(H \ge h) \E_{\pi_\infty} \left[r(\xkh,\akh) \rvert \calH_h \right] \right]\\
&= \E\left[\sum_{h=1}^\infty \ph(H \ge h) \sum_{j=h}^\infty \frac{\ph(j)}{\ph(H\ge h)} \E_{\akh \sim \pi^\star_{k,j}(\calH_h)} \left[r(\xkh,\akh) \rvert \calH_h \right] \right] \\
&= \E\left[\sum_{h=1}^\infty \sum_{j=h}^\infty {\ph(j)} \E_{\akh \sim \pi^\star_{k,j}(\calH_h)} \left[r(\xkh,\akh)  \rvert \calH_h\right] \right] \\
&= \E\left[\sum_{j=1}^\infty \ph(j) \sum_{h=1}^j \E_{\akh \sim \pi^\star_{k,j}(\calH_h)} \left[r(\xkh,\akh)  \rvert \calH_h \right] \right] \\
&= \E\left[\sum_{j=1}^\infty \ph(j) \E_{\pi^\star_{k,j}}\left[\sum_{h=1}^j r(\xkh,\akh) \right] \right] = \sum_{j} V^\star(x_{k,1};\ell) \ph(j)
\end{align*}
Since $\pi_\infty$ is one candidate policy for maximizing infinite horizon sum of discounted rewards, its value function is bounded by the optimal value function.
\begin{equation}\label{eq:optimal-value-ubd}
\sum_{k=1}^T \sum_\ell \ph(\ell) V^\star(x_{k,1};\ell) = \sum_{k=1}^T V^{\pi_\infty}(x_{k,1};\bgamma) \le \sum_{k=1}^T V^\star(x_{k,1};\bgamma)
\end{equation}
Combining equations \ref{eq:reward-equivalence-1} and \ref{eq:optimal-value-ubd} we get that the regret under the new general discounting environment upper bounds the regret under the original environment i.e. $\reg(\bm{\pi}; \ph(\cdot)) \le \reg(\bm{\pi}; \bgamma)$. 
\end{proof}

\subsection{Proof of Lemma~\ref{lem:why-alternative-regret-part-2}}
\begin{proof}
Given a discount factor $\bgamma$, we define the following distribution over episode lengths.
\begin{align*}
    \ph(H=h) = 
         \gamma(h) - \gamma(h+1)
    \end{align*}
This definition implies that $\Pr(H \ge h) = \gamma(h)$ and under this condition lemma~\ref{lem:why-alternative-regret} already shows that $\sum_{k=1}^T \E_{H_k}\left[ V^{\pi_k}(x_{k,1}; H_k)\right] = \sum_{k=1}^T V^{\pi_k}(x_{k,1}; \bgamma)$. Therefore, it remains to upper bound the quantity $\sum_{k=1}^T V^\star(x_{k,1}; \bgamma)$. Let $\pi^\star = (\pi^\star_1,\pi^\star_2,\ldots)$ be the policy that maximizes long-term discounted sum of rewards with respect to the discount factor $\bgamma$. Note that one can also play the policy $\pi^\star$ for the finite horizon setting.
\begin{align*}
    &\sum_\ell \ph(\ell) V^\star(x_{k,1};\ell) \ge \sum_\ell \ph(\ell) V^{\pi^\star}(x_{k,1};\ell) = \sum_{\ell} \ph(\ell) \E_{\pi^\star} \left[ \sum_{j=1}^\ell r(x_{k,j}, a_{k,j}) \right]\\
    &= \E_{\pi^\star} \left[\sum_{j=1}^\infty \sum_{\ell \ge j} (\gamma(\ell) - \gamma(\ell+1)) r(x_{k,j}, a_{k,j})  \right] \\
    &=\E_{\pi^\star} \left[\sum_{j=1}^\infty \gamma(j) r(x_{k,j}, a_{k,j}) \right] = V^\star(x_{k,1}; \bgamma)
\end{align*}
Since the optimal value under the new environment (episode distribution $\ph(\cdot)$) upper bounds the optimal value under the discount factor $\bgamma$, the regret under $\ph(\cdot)$ can only get worse.
\end{proof}

\section{Proof of theorem \ref{thm:main-regret-bound} and Related Corollaries }
We first show that the $Q$ values at each time-step upper bounds the optimal $Q$ values.
\begin{lemma}\label{lem:upper-bound-Q}
With  probability at least $1-\delta$, for all $k \in [T]$, $h \in [N(\Delta)]$, $x \in S$, $a \in \calA$, we have
$$
Q_{k,h}(x,a) \ge Q^\star_h(x,a),  \ V_{k,h}(x) \ge V^\star_{h}(x)
$$
\end{lemma}
\begin{proof}
We prove this statement using induction. First note that we always have $Q_{k,N(\Delta)+1}(x,a) \ge Q^\star_{N(\Delta)+1}(x,a)$ just from the initialization.
Suppose the given claim is true for all $k' \le k-1$ and also $Q_{k,h'}(x,a) \ge Q^\star_k(x,a)$ for all $h' \ge h+1$. We now wish to show that $Q_{k,h}(x,a) \ge Q^\star_k(x,a)$. We can assume that $Q_{k,h}(x,a) \neq Q_{k-1,h}(x,a)$ otherwise the claim is true just from the induction hypothesis. Then we have,

\begin{align*}
Q_{k,h}(x,a) - Q^\star_h(x,a) &= \frac{\gamma(h+1)}{\gamma(h)} \left([\hat{P}_k V_{k,h+1}](x,a) - [P V^\star_{h+1}](x,a)  \right) + \ucb_{k,h}(x,a)\\
&\ge \frac{\gamma(h+1)}{\gamma(h)} \left([(\hat{P}_k - P) V^\star_{h+1}](x,a)  \right) +  \ucb_{k,h}(x,a) \\
\end{align*}
The last inequality follows because from the induction hypothesis $V_{k,h+1}(y) \ge V^{\star}_{h+1}(y)$. We now bound Bernstein's inequality and union bound to bound the first term. Using the fact that $V^\star_{h+1}(\cdot)$ is uniformly bounded by $\Gamma(h+1) / \gamma(h+1)$, we get that the following bound holds with probability at least $1-\delta$ for any $(x,a) \in  \calS \times \calA$, $t \in [T]$, and $h \in [N(\Delta)]$.
$$
\abs{[(\hat{P}_k - P) V^\star_h](x,a) } \le \sqrt{\frac{2\Vm^\star_h(x,a)\ln(SATN(\Delta) /\delta)}{N_k(x,a)}} + \frac{2\Gamma(h) \ln(SATN(\Delta) /\delta)}{3 \gamma(h) N_k(x,a)} 
$$
We now bound the variance term $\Vm^\star_{h+1}(x,a)$ by $(\Gamma(h+1) / \gamma(h+1) )^2$ and substitute the above bound to obtain the following bound on the different between $Q$ values.
\begin{align*}
Q_{k,h}(x,a) - Q^\star_h(x,a) &\ge -\frac{\Gamma(h+1)}{\gamma(h)} \left( \sqrt{\frac{2\ln(SATN(\Delta) /\delta)}{N_k(x,a)}} 
+ \frac{2 \ln(SATN(\Delta) /\delta)}{3 N_k(x,a)} \right) \\ &+  \ucb_{k,h}(x,a) 
\end{align*}
Now it follows from the definition of $\ucb_{k,h}(x,a)$ that the final term is non-negative.
\end{proof}

\subsection{Proof of theorem \ref{thm:main-regret-bound}}
We first provide a formal statement of theorem~\ref{thm:main-regret-bound}.
\begin{theorem}
With probability at least $1-\delta$, Algorithm~\ref{alg:vibi} has the following regret.
\begin{align*}
&\reg(\bm{\pi}; \bm{\gamma}) \le \frac{\Delta T}{\gamma(N(\Delta) + 1)} t(N(\Delta) + 1) +  \max_{h \in [N(\Delta)]} t(h) \frac{\gamma(h+1)}{\gamma(h)} O\left(\sqrt{T N(\Delta) \log(T N(\Delta) / \delta)} \right) \\
&+ \max_{h  \in [N(\Delta)]}  t(h) \frac{ \Gamma(h+1)}{\gamma(h) }  {O}\left( \sqrt{SATN(\Delta) }\right) +  \max_{h \in [N(\Delta)]} t(h) \frac{ (\Gamma(h+1))^2}{\gamma(h) \gamma(h+1)}  \tilde{O}\left( {S^2 A ({N(\Delta)})^{\beta} }\right)
\end{align*}
\end{theorem}
\begin{proof}
At episode $k$, time $h$ the algorithm takes the best action according to $Q_{k,h}(\cdot,\cdot)$. Therefore, $\pi_{k,h}(x_{k,h}) = a_{k,h} \in \argmax_a Q_{k,h}(x_{k,h}, a)$. Moreover, $V^{\pi_k}(\xkh) = Q^{\pi_k}_h(\xkh, \pi_{k,h}(\xkh)) = Q^{\pi_k}(\xkh, \akh)$. 
\begin{align*}
\tilde{\Delta}_{k,h}(x_{k,h}) &= V_{k,h}(x_{k,h}) - V^{\pi_k}_h(x_{k,h}) = Q_{k,h}(\xkh,\akh) - Q^{\pi_k}_h(\xkh,\akh)\\
&\le \frac{\gamma(h+1)}{\gamma(h)}[\hat{P}_k V_{k,h+1} - P V^{\pi_k}_{h+1}](\xkh,\akh)  +  \ucb_{k,h}(\xkh,\akh)
\end{align*}
Now we proceed similarly as the proof of theorem 1 from \cite{AOM17} and derive a recurrence relation for $\tdkh = \tilde{\Delta}_{k,h}(\xkh)$.
\begin{align*}
\tdkh &\le \frac{\gamma(h+1)}{\gamma(h)}  [(\hat{P}_k - P)(V_{h,k+1} - V^\star_{h+1})](\xkh, \akh) \\
&+  \frac{\gamma(h+1)}{\gamma(h)} [P(V_{k,h+1} - V^{\pi_k}_{h+1})](\xkh, \akh) \\
&+ \frac{\gamma(h+1)}{\gamma(h)} [(\hat{P}_k - P)V^\star_{h+1}](\xkh, \akh) \\
&+ \ucb_{k,h}(\xkh,\akh)\\
&= c_{k,h} + \frac{\gamma(h+1)}{\gamma(h)} [P\tilde{\Delta}_{k,h+1}](\xkh, \akh) + e_{k,h} + b_{k,h}
\end{align*}
where in the last line we used the following notations.
\begin{equation}\label{eq:ckh}
c_{k,h} = \frac{\gamma(h+1)}{\gamma(h)}  [(\hat{P}_k - P)(V_{k,h+1} - V^\star_{h+1})](\xkh, \akh) 
\end{equation}
\begin{equation}\label{eq:ekh}
e_{k,h} = \frac{\gamma(h+1)}{\gamma(h)} [(\hat{P}_k - P) V^\star_{h+1} ](\xkh, \akh)  
\end{equation}
\begin{equation}\label{eq:bkh}
b_{k,h} = \ucb_{k,h}(\xkh,\akh)
\end{equation}
Finally, writing $\eps_{k,h}$ as the following error term
\begin{equation}\label{eq:epskh}
\eps_{k,h} = \frac{\gamma(h+1)}{\gamma(h)} [P\tilde{\Delta}_{k,h+1}](\xkh, \akh) - \frac{\gamma(h+1)}{\gamma(h)} \tilde{\Delta}_{k,h+1}(x_{k,h+1})
\end{equation}
we get the following recurrence relation for $\tdkh$.
\begin{equation} \label{eq:rec-delta-1}
\tdkh \le \frac{\gamma(h+1)}{\gamma(h)} \tilde{\delta}_{k,h+1} + c_{k,h} + e_{k,h} + b_{k,h} + \eps_{k,h}
\end{equation}
Let $L=2 \ln(SATN(\Delta) / \delta)$. Then we can apply Bernstein's inequality (lemma~\ref{lem:bernstein}) to conclude that with probability at least $1-\delta$ the following bound holds for each $j \in [N(\Delta)]$ and $k \in [T]$. 
$$
e_{k,j} \le \frac{\gamma(j+1)}{\gamma(j)} \left(\sqrt{\frac{L \var_{y \sim \Pm(\cdot|x_{k,j}, a_{k,j})} V^\star_{j+1}(y) }{N_k(x_{k,j},a_{k,j})}} + \frac{L \max_y V^\star_{j+1}(y) }{3N_{k}(x_{k,j}, a_{k,j})}  \right)
$$
Now instantiating the above bound for $j=h$ and using the fact that $V^\star_{h+1}(y)$ is bounded by $\Gamma(h) / \gamma(h)$ we get the following bound on $e_{k,h}$.
\begin{equation*}
e_{k,h} \le \frac{\Gamma(h+1)}{\gamma(h)}  \left(\sqrt{\frac{L }{N_k(x_{k,h},a_{k,h})}} + \frac{L  }{3N_{k}(x_{k,h}, a_{k,h})}  \right)
\end{equation*}

Now we bound the term $c_{k,h}$ defined in \cref{eq:ckh}.
\begin{align*}
c_{k,h} &= \frac{\gamma(h+1)}{\gamma(h)} \left[ (\hat{P}_k - P)(V_{h,k+1} - V^\star_{h+1})\right](\xkh, \akh)\\
&= \frac{\gamma(h+1)}{\gamma(h)} \sum_y \left(\hat{P}_k(y|\xkh, \akh) - P(y|\xkh,\akh) \right) (V_{k,h+1}(y) - V^\star_{h+1}(y))
\end{align*}
We now substitute two bounds. First we can use Bernstein inequality to bound the  difference between estimated probability and actual probability as
$$
\abs{\hat{P}_k(y|\xkh, \akh) - P(y|\xkh,\akh) } \le \sqrt{\frac{2 P(y|\xkh,\akh) (1-P(y|\xkh,\akh))L}{N_k(\xkh,\akh)} } + \frac{2L }{3 N_k(\xkh,\akh) }
$$
Second we bound $V_{k,h+1}(y) - V^\star_{h+1}(y) \le V_{k,h+1}(y) - V^{\pi_k}_{h+1}(y) \le \tilde{\Delta}_{k,h+1}(y)$. This gives us the following bound on $c_{k,h}$.
\begin{align*}
c_{k,h} &\le  \frac{\gamma(h+1)}{\gamma(h)} \sum_y \sqrt{\frac{2 P(y|\xkh,\akh) L}{N_k(\xkh,\akh)} } \tilde{\Delta}_{k,h+1}(y) +  \frac{\gamma(h+1)}{\gamma(h)} \sum_y \frac{2L }{3 N_k(\xkh,\akh) } \tilde{\Delta}_{k,h+1}(y) \\
&\le \frac{\gamma(h+1)}{\gamma(h)} \sqrt{2L} \underbrace{\sum_y \sqrt{\frac{P(y|\xkh,\akh) }{N_k(\xkh,\akh)} } \tilde{\Delta}_{k,h+1}(y) }_{:= T_1} + \frac{2SL\Gamma(h+1)}{3\gamma(h) N_k(\xkh,\akh) }
\end{align*}
We now bound the term labelled $T_1$. Let 
\begin{equation}\label{eq:defn-ykh}
[y]_{k,h} = \set{y : P(y|\xkh, \akh) N_k(\xkh,\akh) \ge 2 (h+1)^{2\beta} L \left(\frac{\Gamma(h+1)}{\gamma(h+1)} \right)^2}.
\end{equation}
Here we choose a threshold that is different than \cite{AOM17}. In fact, we require the additional $\textrm{poly}(h)$ factor to ensure convergence of the final recurrence relation. Moreover, the threshold is parameterized by the parameter $\beta$, and can be chosen based on the particular discount factor.
\begin{equation*}
T_1 = \sum_{y \in [y]_{k,h} } \sqrt{\frac{P(y|\xkh,\akh) }{N_k(\xkh,\akh)} } \tilde{\Delta}_{k,h+1}(y) + \sum_{y \notin [y]_{k,h} } \sqrt{\frac{P(y|\xkh,\akh) }{N_k(\xkh,\akh)} } \tilde{\Delta}_{k,h+1}(y) 
\end{equation*}
The first term can be bounded as 
\begin{align*}
\sum_{y \in [y]_{k,h} } \sqrt{\frac{P(y|\xkh,\akh) }{N_k(\xkh,\akh)} } \tilde{\Delta}_{k,h+1}(y) &= \hat{\eps}_{k,h} + \sqrt{\frac{P(x_{k,h+1}|\xkh,\akh) }{N_k(\xkh,\akh)} } \tilde{\Delta}_{k,h+1}(x_{k,h+1})\one\set{x_{k,h+1} \in [y]_{k,h} } \\ &\le  \hat{\eps}_{k,h}  +  \frac{\gamma(h+1)}{ \sqrt{2L} (h+1)^{\beta} \Gamma(h+1)}  \tilde{\delta}_{k,h+1}
\end{align*}
where $\hat{\eps}_{k,h}$ is an error term defined as follows.
\begin{align}
\hat{\eps}_{k,h} &= \sum_{y \in [y]_{k,h} } \sqrt{\frac{P(y|\xkh,\akh) }{N_k(\xkh,\akh)} } \tilde{\Delta}_{k,h+1}(y) \nonumber \\ &- \sqrt{\frac{P(x_{k,h+1}|\xkh,\akh) }{N_k(\xkh,\akh)} } \tilde{\Delta}_{k,h+1}(x_{k,h+1})\one\set{x_{k,h+1} \in [y]_{k,h} } \label{eq:defn-hat-epskh}
\end{align}
It can be checked that $\hat{\eps}_{k,h}$ is actually a Martingale difference sequence. The second term of $T_1$ can be bounded as
\begin{align*}
\sum_{y \notin [y]_{k,h} } \sqrt{\frac{P(y|\xkh,\akh) }{N_k(\xkh,\akh)} } \tilde{\Delta}_{k,h+1}(y) &\le \sum_{y \notin [y]_{k,h} } \frac{\sqrt{P(y|\xkh,\akh) N_k(\xkh,\akh)} }{N_k(\xkh,\akh)} \tilde{\Delta}_{k,h+1}(y) \\
&\le \frac{ \sqrt{2 L} (h+1)^{\beta} S}{N_k(\xkh,\akh)} \left(\frac{\Gamma(h+1)}{\gamma(h+1)} \right)^2
\end{align*}
Substituting the bound on $T_1$ in the upper bound on $c_{k,h}$ we get the following bound.
\begin{align*}
c_{k,h} &\le \frac{\gamma(h+1)}{\gamma(h)} \sqrt{2L} \left( \hat{\eps}_{k,h}  +  \frac{\gamma(h+1)}{ \sqrt{2L} (h+1)^{\beta} \Gamma(h+1)}  \tilde{\delta}_{k,h+1} + \frac{ \sqrt{2L}(h+1)^{\beta}  S}{N_k(\xkh,\akh)} \left(\frac{\Gamma(h+1)}{\gamma(h+1)} \right)^2\right) \\
&+ \frac{2SL\Gamma(h+1)}{3\gamma(h) N_k(\xkh,\akh) } \\
&\le \frac{\gamma(h+1)}{\gamma(h)} \sqrt{2L} \hat{\eps}_{k,h} + \frac{(\gamma(h+1))^2 }{\gamma(h) \Gamma(h+1) (h+1)^{\beta}} \tilde{\delta}_{k,h+1} + \underbrace{\frac{(h+1)^{\beta}(\Gamma(h+1))^2}{\gamma(h)\gamma(h+1)} \frac{2 LS}{N_k(\xkh,\akh)} }_{:= f_{k,h} }
\end{align*}

Substituting the previous upper bound on $c_{h,k}$ in \cref{eq:rec-delta-1} and writing $\bar{\eps}_{k,h} = \gamma(h+1) / \gamma(h) \hat{\eps}_{k,h}$ we get the final recurrence relation for $\tdkh$.
\begin{equation}\label{eq:rec-final}
\tdkh \le \frac{\gamma(h+1)}{\gamma(h)} \left(1 + \frac{\gamma(h+1)}{(h+1)^{\beta} \Gamma(h+1)} \right) \tilde{\delta}_{k,h+1} + \sqrt{2L}\bar{\eps}_{k,h} + e_{k,h} + b_{k,h}+\eps_{k,h}+f_{k,h}
\end{equation}

Then expanding the recurrence relation~\cref{eq:rec-final} from $h=1$ to $h=N(\Delta)$ we get the following equation.
\begin{equation}
\tilde{\delta}_{k,1} \le t(N(\Delta)+1) \tilde{\delta}_{k,N(\Delta)+1} + \sum_{h=1}^{N(\Delta)} t(h) \left( \sqrt{2L}\bar{\eps}_{k,h} + e_{k,h} + b_{k,h}+\eps_{k,h}+f_{k,h}\right)
\end{equation}
We now 
use the fact that $\tilde{\delta}_{k,N(\Delta) + 1} \le \Delta / \gamma(N(\Delta) + 1)$ and obtain the following bound on $\tilde{\delta}_{k,1}$.
$$
\tilde{\delta}_{k,1} \le \frac{ \Delta}{\gamma(N(\Delta)+1) } t(N(\Delta) +1) +  \sum_{h=1}^{N(\Delta)} t(h) \left( \sqrt{2L}\bar{\eps}_{k,h} + e_{k,h} + b_{k,h}+\eps_{k,h}+f_{k,h}\right)
$$
Summing over the $T$ episodes the final bound on regret is given as.
\begin{equation}\label{eq:regret-bound-total}
\reg(T) \le \sum_{k=1}^T \tilde{\delta}_{k,1} \le \frac{\Delta T }{\gamma(N(\Delta) + 1)} t(N(\Delta) + 1)  + \sum_{k=1}^T  \sum_{h=1}^{N(\Delta)} t(h) \left( \sqrt{2L}\bar{\eps}_{k,h} + e_{k,h} + b_{k,h}+\eps_{k,h}+f_{k,h}\right)
\end{equation}

We now bound the second term. First consider the exploration bonus term $b_{k,h}$. 
\begin{align*}
\sum_{k=1}^T \sum_{h=1}^{N(\Delta)} t(h) b_{k,h} &= \sum_{k=1}^T \sum_{h=1}^{N(\Delta)} t(h) \frac{\Gamma(h+1)}{\gamma(h)} \frac{L}{\sqrt{N_k(\xkh,\akh)}} \\
&\le L \max_{h \in [N(\Delta)]} t(h) \frac{\Gamma(h+1)}{\gamma(h)}  \sum_{k,h} \frac{1}{\sqrt{N_k(\xkh,\akh)}}
\end{align*}
Using a simple counting argument we can bound the last summation as follows. Here we use the fact that only the first $N(\Delta)$ steps of each episode are used to update the $N_k(\cdot,\cdot)$ counters. 
\begin{align*}
\sum_{k=1}^T \sum_{h=1}^{N(\Delta)} \frac{1}{\sqrt{N_k(\xkh,\akh)}} &= \sum_{x,a} \sum_{n=1}^{N_T(x,a)} \frac{1}{\sqrt{n}} \le 2 \sum_{x,a} \sqrt{N_{k}(x,a)} \\ &\le 2 \sqrt{SA}\sqrt{\sum_{x,a} N_T(x,a) } = 2\sqrt{SATN(\Delta)}
\end{align*}
This result gives us the following bound on the sum of $b_{k,h}$terms.

\begin{equation}\label{eq:bound-bkh}
\sum_{k=1}^T \sum_{h=1}^{N(\Delta)} t(h) b_{k,h} = L \max_{h \in [N(\Delta)]} t(h) \frac{\Gamma(h+1)}{\gamma(h)}  {O}\left( \sqrt{SATN(\Delta) }\right)
\end{equation}
By a similar argument we can prove the following bounds on the $e_{k,h}$ and $f_{k,h}$ terms.
\begin{equation}\label{eq:bound-ekh}
\sum_{k=1}^T \sum_{h=1}^{N(\Delta)} t(h) e_{k,h} = L \max_{h \in [N(\Delta)]} t(h) \frac{\Gamma(h+1)}{\gamma(h)} {O}\left( \sqrt{SATN(\Delta) }\right)
\end{equation}

\begin{equation}\label{eq:bound-fkh}
\sum_{k=1}^T \sum_{h=1}^{N(\Delta)} t(h) f_{k,h} \le L \max_{h \in [N(\Delta)]} t(h) \frac{ (\Gamma(h+1))^2}{\gamma(h) \gamma(h+1)} {O}\left( {S^2 A (N(\Delta))^{\beta} \log(TN(\Delta) )}\right)
\end{equation}

We  now consider the martingale differences term $\varepsilon_{k,h}$ and $\bar{\varepsilon}_{k,h}$. If we write $\calF_{k,h}$ to denote the sigma-algebra generated by the actions, and states until step $h$ of episode $k$, then we have $\E[\varepsilon_{k,h}| \calF_{k,h}] = 0$. Moreover, each $\varepsilon_{k,h}$ is bouned by $O\left(\frac{\Gamma(h+1) }{ \gamma(h)}\right)$. So we can define a new set of martingales $\tilde{\varepsilon}_{k,h} = \epsilon_{k,h} \frac{\gamma(h)}{\Gamma(h+1)}$ so that $\abs{\tilde{\varepsilon}_{k,h}} \le 1$. Now, we can apply the Azuma-Hoeffding inequality to get the following result with probability at least $1-\delta$.
\begin{align}
    \sum_{k=1}^T \sum_{h =1}^{N(\Delta)} t(h) \varepsilon_{k,h} &\le \max_{h \in [N(\Delta)]} t(h) \frac{\Gamma(h+1)}{\gamma(h)}  \sum_{k=1}^T \sum_{h =1}^{N(\Delta)} \abs{\tilde{\varepsilon}_{k,h}}\nonumber \\
    &\le\max_{h \in [N(\Delta)]} t(h) \frac{\Gamma(h+1)}{\gamma(h)}  O\left(\sqrt{T N(\Delta) \log(T N(\Delta) / \delta)} \right) \label{eq:bound-epskh}
\end{align}

Now recall that we defined $\bar{\varepsilon}_{k,h} = \frac{\gamma(h+1)}{\gamma(h)}\hat{\varepsilon}_{k,h}$ where $\hat{\varepsilon}_{k,h}$ was defined in equation~\ref{eq:defn-hat-epskh}. This gives us the following expression for $\bar{\varepsilon}_{k,h}$.
\begin{align*}
\bar{\eps}_{k,h} &= \frac{\gamma(h+1)}{\gamma(h)}\sum_{y \in [y]_{k,h} } \sqrt{\frac{P(y|\xkh,\akh) }{N_k(\xkh,\akh)} } \tilde{\Delta}_{k,h+1}(y) \\ &- \frac{\gamma(h+1)}{\gamma(h)} \sqrt{\frac{P(x_{k,h+1}|\xkh,\akh) }{N_k(\xkh,\akh)} } \tilde{\Delta}_{k,h+1}(x_{k,h+1})\one\set{x_{k,h+1} \in [y]_{k,h} } 
\end{align*}
It can be easily verified that $\E[\eps_{k,h} | \calF_{k,h}] = 0$. We now establish an upper bound on $\bar{\eps}_{k,h}$. Consider the first term in the definition of $\bar{\eps}_{k,h}$.
\begin{align*}
    &\frac{\gamma(h+1)}{\gamma(h)}\sum_{y \in [y]_{k,h} } \sqrt{\frac{P(y|\xkh,\akh) }{N_k(\xkh,\akh)} } \tilde{\Delta}_{k,h+1}(y) \\&= \frac{\gamma(h+1)}{\gamma(h)}\sum_{y \in [y]_{k,h} } {\frac{P(y|\xkh,\akh) }{\sqrt{P(y|\xkh, \akh) N_k(\xkh,\akh)} } } \tilde{\Delta}_{k,h+1}(y) \\
    &\le \frac{1}{\sqrt{2L} (h+1)^{\beta}} \frac{\gamma(h+1)}{\gamma(h)} \frac{\gamma(h+1)}{\Gamma(h+1)} \frac{\Gamma(h+1)}{\gamma(h+1)} \le \frac{1}{2\sqrt{L}} \frac{\gamma(h+1)}{\gamma(h)}
\end{align*} 
The last inequality uses the definition of $[y]_{k,h}$ (\cref{eq:defn-ykh}).The second term of $\bar{\eps}_{k,h}$ can be bounded similarly. Therefore, again using Azuma-Hoeffding inequality, we get the following result with probability at least $1-\delta$.
\begin{equation}\label{eq:bound-barepskh}
    \sum_{k=1}^T \sum_{h =1}^{N(\Delta)} \sqrt{2L} t(h) \bar{\eps}_{k,h} \le \max_{h  \in [N(\Delta)]} t(h) \frac{\gamma(h+1)}{\gamma(h) } O\left(\sqrt{T N(\Delta) \log(T N(\Delta) / \delta)} \right)
\end{equation}

Substituting equations \ref{eq:bound-bkh}, \ref{eq:bound-ekh}, \ref{eq:bound-fkh}, \ref{eq:bound-epskh}, \ref{eq:bound-barepskh} in equation \ref{eq:regret-bound-total} we get the final bound on regret.
\end{proof}

\begin{lemma}[Bernstein's Inequality]\label{lem:bernstein}
Let $Z_1,\ldots,Z_n$ be i.i.d. random variables with values bounded by $H$ and let $\delta > 0$. Then with probability at least $1-\delta$ we have
$$
\E Z_1 - \frac{1}{n} \sum_{i=1}^n Z_i \le \sqrt{\frac{2 \var(Z_1) \ln(2/\delta)}{n}} + {\frac{2 H \ln(2/\delta)}{3n} }
$$
\end{lemma}

\subsection{Proof of Corollary~\ref{cor:poly-decay}}

\begin{proof}

Consider the discount factor $\gamma(h) = h^{-p}$ for $p \ge 2$. Recall that we defined $\Gamma(h) = \sum_{j \ge h} \gamma(j)$. We will use the following bound on $\Gamma(h) \in \left[\frac{h^{-p+1}}{p-1}, h^{-p}\left(1 + \frac{h}{p-1}\right) \right]$. We substitute $\beta = p-1$ and bound various $\gamma$-dependent constants appearing in the regret bound.
$$
\frac{t(N(\Delta) + 1)}{\gamma(N(\Delta) + 1)} = \prod_{j=2}^{N(\Delta) + 1} \left( 1 + \frac{\gamma(j)}{{j^{\beta}} \Gamma(j)}\right) \le \prod_{j=2}^{N(\Delta) + 1}\left(1 + \frac{p-1}{j^{1+\beta}}\right) \le e^{(p-1)\sum_j j^{-p}} \le e
$$
\begin{align*}
\max_{h \in [N(\Delta)]} t(h) \frac{ \Gamma(h+1)}{\gamma(h) } = \max_{h \in [N(\Delta)]} \Gamma(h+1) \prod_{j=2}^{h } \left( 1 + \frac{\gamma(j)}{{j^{\beta}} \Gamma(j)}\right) \le e \max_{h \in [N(\Delta)]} h^{-p} \left(1 + \frac{h}{p-1} \right) = e 
\end{align*}
The last inequality follows because the final term is a decreasing function of $h$ for $p > 1$.
\begin{align*}
\max_{h \in [N(\Delta)]} t(h) \frac{ (\Gamma(h+1))^2}{\gamma(h) \gamma(h+1)} &= \max_{h \in [N(\Delta)]} \frac{(\Gamma(h+1))^2}{\gamma(h+1)} \prod_{j=2}^{h } \left( 1 + \frac{\gamma(j)}{{j^{\beta}} \Gamma(j)}\right) \\ &\le e \max_{h \in [N(\Delta)]} (h+1)^{-p}\left(1 + \frac{h+1}{p-1} \right)^2 \le e
\end{align*}
Here we use $p \ge 2$ to conclude that the term inside max is non-increasing. We now bound the term $N(\Delta)$. Recall that $N(\Delta)$ is the value of $h$ such that $\Gamma(h)$ is bounded from above by $\Delta$. Using the lower bound of $h^{-p+1}/(p-1)$ we get $N(\Delta)$ should be at least $(\Delta(p-1))^{-1/(p-1)}$.
Substituting the value of $N(\Delta)$ and the bound on the three constants in the expression for regret, we get the following bound on regret.
\begin{align*}
    \reg(T) \le e\Delta T &+ e\left(\Delta(p-1) \right)^{-\frac{1}{2(p-1)} }\left[\tilde{O}(\sqrt{SAT})  + \tilde{O}(\sqrt{T}) \right]\\
    &+ e\left(\Delta(p-1) \right)^{-\frac{1/2}{p-1}}\tilde{O}(S^2 A)
\end{align*}
For $T \ge O(S^3 A) $ then the last term is dominated by the other terms in the expression above. Finally, substituting $\Delta = T^{-\frac{p-1}{2p-1} } (p-1)^{-\frac{1}{2p-1}}$ we get the following bound on regret.
$$
\reg(T) \le  (p-1)^{-\frac{1}{2p-1}}S^{1/2} A^{1/2} T^{\frac{p}{2p-1}} = O\left( S^{1/2} A^{1/2} T^{\frac{p}{2p-1}} \right)
$$
The last equality uses the fact that $p^{-O(1/p)} = O(1)$ for $p \ge 2$. 

\underline{$1 < p < 2$}: We now substitute $\beta = p-1$. Proceeding similarly as earlier, we can establish the following inequalities.
$$
\frac{t(N(\Delta) + 1)}{\gamma(N(\Delta) + 1)} \le e^{(p-1)/\beta} = e
$$
$$
\max_{h \in [N(\Delta)]} t(h) \frac{ \Gamma(h+1)}{\gamma(h) } \le e^2 \max_{h \in [N(\Delta)]} h^{-p}\left( 1  + \frac{h}{p-1}\right) \le e \frac{p}{p-1}
$$
The third term is different for the case of $1 < p < 2$.
\begin{align*}
\max_{h \in [N(\Delta)]} t(h) \frac{ (\Gamma(h+1))^2}{\gamma(h) \gamma(h+1)} &\le e \max_{h \in [N(\Delta)]} (h+1)^{-p}\left(1 + \frac{h+1}{p-1} \right)^2 \\
&\le e \max_{h \in [N(\Delta)]} (h+1)^{-p}\left(1 + 2 \frac{h+1}{p-1} + \frac{(h+1)^2}{(p-1)^2} \right) \\
&\le e\left( \frac{p}{2(p-1)} + \frac{(1+N(\Delta))^{2-p}}{(p-1)^2}\right) \le \frac{2e}{p-1} (N(\Delta))^{2-p}
\end{align*}
We now substitute the inequalities above and use $N(\Delta) = (\Delta(p-1))^{-1/(p-1)}$.
\begin{align*}
\reg(T) &\le e \Delta T + e \frac{p}{p-1} \sqrt{SAT}(\Delta (p-1))^{-1/2(p-1)} \\
&+ \frac{2e}{p-1} (\Delta(p-1))^{-1/(p-1)} S^2 A \log(T) + e \sqrt{T \log T} (\Delta (p-1))^{-1/2(p-1)} 
\end{align*}
We substitute $\Delta = T^{-\frac{p-1}{2p-1} } $ to get the following bound on regret.
\begin{align*}
\reg(T) &\le e T^{p/(2p - 1)} + e \frac{p}{p-1} (p-1)^{-1/2(p-1)} T^{p/(2p-1)} \sqrt{SA}\\
&+ \frac{2e}{p-1} (p-1)^{-1/(p-1)} T^{1/(2p-1)} S^2 A \log T + e (p-1)^{-1/2(p-1)} T^{p/(2p-1)} \sqrt{\log T}
\end{align*}
If $T \ge O\left( (S^3 A)^{(2p-1)/2(p-1)}\right)$ then the first term dominates the third term and we get the following bound on regret.
$$
\reg(T) \le 4e \frac{p}{p-1} T^{\frac{p}{2p-1}} \sqrt{SA \log T} (p-1)^{-1/(p-1)}
$$
\end{proof}

\subsection{Proof of Corollary~\ref{cor:exp-decay}}
\begin{proof}
For discount factor $\gamma(h) = \gamma^{h-1}$, it is easy to check that $\Gamma(h) = \gamma^{h-1}/ (1-\gamma)$. We will substitute $\beta = 3/2$ and use the following inequality.
$$
\prod_{j=2}^h \left(1 + \frac{\gamma(j)}{{j^{\beta}} \Gamma(j)} \right) = \prod_{j=2}^h \left(1 + \frac{1-\gamma}{{j^{3/2}}} \right) = \exp \left((1-\gamma) \sum_{j=2}^h \frac{1}{{j^{3/2}} }\right) \le e^{1-\gamma}
$$
This gives the following bounds on the three constants.
$$
\frac{t(N(\Delta) + 1)}{\gamma(N(\Delta) + 1)} = \prod_{j=2}^{N(\Delta) + 1} \left( 1 + \frac{\gamma(j)}{{j^{\beta}} \Gamma(j)}\right) \le e^{1-\gamma}
$$
\begin{align*}
\max_{h \in [N(\Delta)]} t(h) \frac{ \Gamma(h+1)}{\gamma(h) } &= \max_{h \in [N(\Delta)]} \Gamma(h+1) \prod_{j=2}^{h } \left( 1 + \frac{\gamma(j)}{{j^{\beta}} \Gamma(j)}\right) \\&\le \max_{h \in [N(\Delta)]} \frac{\gamma^h}{1-\gamma} e^{1-\gamma } \le \frac{\gamma}{1-\gamma}  e^{1-\gamma} 
\end{align*}
Now let us consider the second constant.
\begin{align*}
\max_{h \in [N(\Delta)]} t(h) \frac{ (\Gamma(h+1))^2}{\gamma(h) \gamma(h+1)} &= \max_{h \in [N(\Delta)]} \frac{(\Gamma(h+1))^2}{\gamma(h+1)} \prod_{j=2}^{h } \left( 1 + \frac{\gamma(j)}{{j^{\beta}} \Gamma(j)}\right) \\ &\le \max_{h \in [N(\Delta)]} \frac{\gamma^h}{(1-\gamma)^2} e^{1-\gamma} \le \frac{\gamma}{(1-\gamma)^2} e^{1-\gamma} 
\end{align*}
Since $\Gamma(h) = \gamma^h/(1-\gamma)$, it is easy to show that $N(\Delta) = \frac{\log(1/\Delta(1-\gamma))}{\log(1/\gamma)}$. Substituting the bounds on the constants and the bound on $N(\Delta)$ we get the following bound on the regret.
\begin{align*}
\reg(T) &\le e^{1-\gamma} \Delta T + \frac{e^{1-\gamma} }{1-\gamma} \sqrt{\frac{\log(1/\Delta(1-\gamma))}{\log(1/\gamma)}} \left( \tilde{O}(\sqrt{SAT}) + \tilde{O}(\sqrt{T}) \right) \\
&+ \frac{e^{1-\gamma} }{(1-\gamma)^2} \left( \frac{\log(1/\Delta(1-\gamma))}{\log(1/\gamma)}\right)^{\beta} \tilde{O}(S^2 A)
\end{align*}
If $T \ge \frac{1}{(1-\gamma)^2} (N(\Delta))^{2} \tilde{O}(S^3 A)$ then the last term in the expression is dominated by the other terms. We further use the inequality $\log(1/\gamma) \ge 1-\gamma$ to simplify the regret expression.
$$
\reg(T) \le e^{1-\gamma} \Delta T + \frac{1}{(1-\gamma)^{1.5}} \sqrt{\log(1/\Delta(1-\gamma))} \tilde{O}(\sqrt{SAT})
$$to get a regret bound of $\tilde{O}\left(\frac{\sqrt{SAT}}{(1-\gamma)^{1.5}} \right)$. Notice that we need to satisfy the following bound on $T$.
$$
T \ge \frac{1}{(1-\gamma)^2} \left(\frac{\log(1/\Delta(1-\gamma))}{\log(1/\gamma)} \right)^{2} \tilde{O}(S^3 A)
$$
For $\Delta = T^{-1}/(1-\gamma)$ the above bound is equivalent to the condition $T \ge \tilde{O} \left(S^3 A / (1-\gamma)^{4}\right)$
\end{proof}

\section{Proof of theorem~\ref{thm:bound-alg-est} and Related Corollaries}
We first provide a formal statement of theorem~\ref{thm:bound-alg-est}.
\begin{theorem}
When run with  horizon length $H^\star$, algorithm~\ref{alg:estimate-gamma} has the following regret bound with probability at least $1-\delta$
\begin{align*}
&\reg(\pi;\bgamma) \le \min_{L \in [T]}\left(T \Gamma(L+1) + 2L \log(T) \sqrt{T} )\right)
+\Gamma(H^\star + 1) T + \tilde{O}\left(\sqrt{T}\right) \\
&+ \log(T) \frac{t(H^\star + 1)}{\gamma(H^\star + 1)}g(H^\star + 1)
    +\max_{h \in[H^\star]} \frac{t(h)}{\gamma(h)}g(h) \Gamma(h+1) \left( 1 + \frac{O(T^{-1/4})}{\Gamma(h+1)}\right) \tilde{O}\left(\sqrt{SATH^\star}\right)\\
&+  \log^2(T) \log(T H^\star) \max_{h\in[H^\star]} t(h) \frac{(\Gamma(h+1))^2}{\gamma(h)\gamma(h+1)}g(h)  \left(1 + \frac{O(T^{-1/4})}{\Gamma(h+1)} \right)^2 O\left(S^2 A (H^\star)^{\beta} \right)\\
&+  \log^{5/2}(T) \max_{h \in [H^\star]}\frac{t(h)}{\gamma(h)}\gamma(h+1)g(h)\left(1 + \frac{O( T^{-1/4})}{\gamma(h+1)} \right) O\left(\sqrt{TH^\star \log(TH^\star/\delta)} \right)
\end{align*}
where $g(h) = \exp\left\{O\left(\sum_{k=2}^h \frac{T^{-1/4}}{\gamma(k) + k^{\beta} \Gamma(k)} \right) \right\}$.
\end{theorem}

\begin{proof}
Since each $\hat{\gamma}_j$ is the empirical distribution function, by the Dvoretzky–Kiefer–Wolfowitz inequality \cite{DKW56, Massart90} and a union bound over the $\log(T/B)$ blocks the following result holds.
\begin{equation*}
    \Pr \left(\forall j \sup_h \abs{\hat{\gamma}_j(h) - \gamma(h)} > \sqrt{\frac{\log\left(2 \log(T/B)/\delta \right)}{2^j B}} \right) \le  \delta
\end{equation*}
For the remainder of the proof we will assume $\norm{\hat{\gamma}_j - \gamma}_\infty \le \eps_j$ for $\eps_j = \sqrt{\frac{\log\left(2 \log(T/B)/\delta \right)}{2^j B}}$. Now consider any $j > 1$. Since algorithm~\ref{alg:vibi} is run for $2^j B$ episodes with estimate $\hat{\gamma}_j$ we have
$$
\sum_{k\in B_j} V^\star(x_{k,1}; \hbgamma_j) - V^{\pi_k}(x_{k,1};\hbgamma_j) \le \calR(H^\star, \hat{\Delta}_j, 2^j B, \hbgamma_j)
$$
Here we write $B_j$ to denote the episodes in block $j$ and $\calR(H, \Delta, T, \hbgamma_j)$ is the regret bound derived in theorem \ref{thm:main-regret-bound} with $N(\Delta)=H$ and discount factor $\hbgamma_j$. We establish a lower bound on $V^\star(x_{k,1}; \hbgamma_j)$.
\begin{align*}
    V^\star(x_{k,1}; \bgamma) &= \E_{\pi^\star(\bgamma)}\left[\sum_{h=1}^\infty \gamma(h) r(\xkh, \akh)   \right] \\ &\le \E_{\pi^\star(\bgamma)}\left[\sum_{h=1}^L \gamma(h) r(\xkh, \akh)   \right] + \Gamma(L+1) \\
    &\le \E_{\pi^\star(\bgamma)} \left[\sum_{h=1}^L \hat{\gamma}_j(h) r(\xkh, \akh)   \right] + \varepsilon_j L + \Gamma(L+1) \\
    &\le V^\star(x_{k,1}; \hbgamma_j) + \varepsilon_j L + \Gamma(L+1)
\end{align*}
The first inequality uses the fact that rewards are bounded by $1$ and $\Gamma(L+1) = \sum_{h \ge L+1} \gamma(h)$. The second inequality uses $\max_h\abs{\hat{\gamma}_j(h) - \gamma(h)} \le \eps_j$. This gives us the following lower bound.
\begin{equation}\label{eq:lbd-v-star}
    V^\star(x_{k,1};\hbgamma_j) \ge V^\star(x_{k,1};\bgamma) - \eps_j L - \Gamma(L+1)
\end{equation}
We now derive an upper bound on $V^{\pi_k}(x_{k,1}; \hbgamma_j)$.
\begin{align}
    V^{\pi_k}(x_{k,1};\bgamma) &= \E_{\pi^k} \left[\sum_{h=1}^\infty \gamma(h) r(\xkh,\akh) \right] \nonumber \\
    &\ge \E_{\pi_k} \left[\sum_{h=1}^L \gamma(h) r(\xkh,\akh)  \right] - \Gamma(L+1)\nonumber \\
    &\ge \E_{\pi_k} \left[\sum_{h=1}^L \hat{\gamma}(h) r(\xkh,\akh)  \right] - \eps_j L - \Gamma(L+1) \nonumber\\
    &= V^{\pi_k}(x_{k,1};\hbgamma_j) - \E_{\pi_k}\left[\sum_{h=L+1}^\infty \hat{\gamma}(h) r(\xkh,\akh) \right] - \eps_j L - \Gamma(L+1) \nonumber\\
    &\ge V^{\pi_k}(x_{k,1};\hbgamma_j) - \hat{\Gamma}_j(L+1) - \eps_j L - \Gamma(L+1) \label{eq:bound-temp-vpik}
\end{align}

We now use lemma~\ref{lem:bound-Gamma} with $D=T$ and a union bound over the $\log(T/B)$ blocks to get the following bound for all $j \in [\log(T/B)]$ and $h \in [T]$. 
$$
\abs{\hat{\Gamma}_j(h) - \Gamma(h)} \le \sqrt{\frac{\log T + \log(\log(T/B)/\delta)}{\abs{B_j}}} \le \left(1 + \sqrt{\log T} \right) \eps_j
$$
Substituting the above bound in equation~\ref{eq:bound-temp-vpik} and rearranging we get the following bound.
\begin{align}\label{eq:ubd-vpik}
V^{\pi_k}(x_{k,1}; \hbgamma) \le V^{\pi_k}(x_{k,1};\bgamma) +  \eps_j \left(L + 1 + \sqrt{\log T} \right)
\end{align}
Using equations \ref{eq:lbd-v-star} and \ref{eq:ubd-vpik} we can bound the regret of algorithm~\ref{alg:estimate-gamma} during the episodes of block $B_j$.
\begin{align*}
&\sum_{k \in B_j} V^\star(x_{k,1}; \bgamma) - V^{\pi_k}(x_{k,1};\bgamma) \\&\le \sum_{k\in B_j} V^\star(x_{k,1};\hbgamma) - V^{\pi_k}(x_{k,1};\hbgamma) + \Gamma(L+1) + \eps_j \left(2L + 1 + \sqrt{\log T}\right) \\
&\le \calR(H^\star, \hat{\Delta}_j, 2^j B, \hbgamma_j) +  \abs{B_j} \Gamma(L+1) +\abs{B_j}  \eps_j \left(2L + 1 + \sqrt{\log T}\right)
\end{align*}
We use the fact that $\eps_j = \sqrt{\frac{\log(2\log(T/B)/\delta}{\abs{B_j}}}$ and sum over all $j = 1,\ldots, \log(T/B)$.
\begin{align*}
    \sum_{t=1}^T V^\star(x_{t,1};\bgamma) - V^{\pi_t}(x_{t,1};\bgamma) &\le \sum_{j=1}^{\log(T/B)}  \calR(H^\star, \hat{\Delta}_j, 2^j B, \hbgamma_j) \\&+  \sum_j \abs{B_j} \Gamma(L+1) + \sum_{j=1}^{\log(T/B)} \sqrt{\log(2\log(T/B)/\delta)}{O}\sqrt{\abs{B_j}}  \left( 2L + 1 + \sqrt{\log T}\right) 
\end{align*}
Now we use $\sum_j \abs{B_j} = T$ and $\sum_j \sqrt{\abs{B_j}} \le \sqrt{\log(T/B)} \sqrt{\sum_j \abs{B_j}} = {O}\left(\sqrt{T\log(T/B)} \right)$, and $B=\sqrt{T}\log T \log (\log(T)/\delta)$ to get the following bound on regret.
\begin{align*}
    \sum_{t=1}^T V^\star(x_{t,1};\bgamma) - V^{\pi_t}(x_{t,1};\bgamma) &\le \sum_{j=1}^{\log(T/B)}  \calR(H^\star, \hat{\Delta}_j, 2^j B, \hbgamma_j) \\&+ \min_{L \in [T]}\left(T \Gamma(L+1) + 2L \log(T) \sqrt{T} )\right)
\end{align*}
We now bound the first term by substituting the definition of of $\calR(H^\star, \hat{\Delta}_j, \abs{B_j}, \hbgamma_j)$.
\begin{align}\label{eq:reg-expr-est}
\begin{split}
    \sum_{j=1}^{\log (T/B)} \calR(H^\star, \hat{\Delta}_j, \abs{B_j}, \hbgamma_j) &\le \sum_{j=1}^{\log (T/B)} \frac{\hat{\Delta}_j \abs{B_j}}{\hat{\gamma}_j(H^\star + 1)} \hat{t}_j(H^\star + 1) \\
    &+ \sum_{j=1}^{\log (T/B)} \max_{h  \in [H^\star]}  \hat{t}_j(h) \frac{ \hat{\Gamma}_j(h+1)}{\hat{\gamma}_j(h) }  {O}\left( \sqrt{SA\abs{B_j} H^\star }\right)\\
&+  \sum_{j=1}^{\log (T/B)}  \max_{h \in [H^\star]} \hat{t}_j(h) \frac{ (\hat{\Gamma}_j(h+1))^2}{\hat{\gamma}_j(h) \hat{\gamma}_j(h+1)}  {O}\left( {S^2 A {H^\star}^{\beta} \log(\abs{B_j} H^\star )}\right)\\
&+  \sum_{j=1}^{\log (T/B)}  \max_{h \in [H^\star]} \hat{t}_j(h) \frac{\hat{\gamma}_j(h+1)}{\hat{\gamma}_j(h)} O\left(\sqrt{\abs{B_j} H^\star \log(\abs{B_j} H^\star / \delta)} \right)
\end{split}
\end{align}

We now bound the four terms separately. We will frequently use the following two identities. 
\begin{equation}\label{eq:bound-sum-eps}
\sum_{j=1}^{\log(T/B)} \eps_j \le \sum_{j=1}^\infty \frac{\eps_1}{(\sqrt{2})^{j-1}} = {O}\left(\sqrt{\frac{\log(\log(T/B) / \delta)}{B} }\right)
\end{equation}

\begin{align*}
    &\sum_{j=1}^{\log(T/B)} \max_{h \in [H^\star]} \frac{\hat{t}_j(h)}{\hat{\gamma}_j(h)} = \sum_{j=1}^{\log(T/B)} \max_{h \in [H^\star]} \prod_{k=2}^{h} \left(1 + \frac{\hat{\gamma}_j(k)}{k^{\beta} \hat{\Gamma}_j(k)} \right) \\
    &\le \sum_{j=1}^{\log(T/B)} \max_{h \in [H^\star]} \prod_{k=2}^{h} \left(1 + \frac{\gamma(k)}{k^{\beta} \Gamma(k)} + \frac{2 \eps_j \sqrt{\log T}}{k^{\beta} \Gamma(k)} \right)\quad \textrm{[By lemma~\ref{lem:ratio-inequality}]}\\
    &= \sum_{j=1}^{\log(T/B)} \max_{h \in [H^\star]} \frac{t(h)}{\gamma(h)}\prod_{k=2}^h \left( 1 + \frac{2 \eps_j \sqrt{\log T}}{\gamma(k) + k^{\beta}\Gamma(k)}\right)  \\
    &\le \log(T/B) \max_{h \in [H^\star]}\frac{t(h)}{\gamma(h)}  \sum_{j=1}^{\log(T/B)}\prod_{k=2}^h \left( 1 + \frac{2 \eps_j \sqrt{\log T}}{\gamma(k) + k^{\beta} \Gamma(k)}\right)  \\
    &\le \log(T/B) \max_{h \in [H^\star]}\frac{t(h)}{\gamma(h)}\sum_{j=1}^{\log(T/B)} \exp\left\{\sum_{k=2}^h \frac{2 \eps_j \sqrt{\log T}}{\gamma(k) + k^{\beta} \Gamma(k)} \right\} \\
    &\le \log(T/B) \max_{h \in [H^\star]}\frac{t(h)}{\gamma(h)} \exp\left\{\sum_{j=1}^{\log(T/B)} \sum_{k=2}^h \frac{2 \eps_j \sqrt{\log T}}{\gamma(k) + k^{\beta} \Gamma(k)} \right\}\\
    &\le \log(T/B) \max_{h \in [H^\star]}\frac{t(h)}{\gamma(h)}  \exp\left\{{O}\left(\sqrt{\frac{\log(\log(T/B) / \delta)}{B} } \sum_{k=2}^h \frac{2  \sqrt{\log T}}{\gamma(k) + k^{\beta} \Gamma(k)}\right) \right\} \quad \textrm{[By eq.~\ref{eq:bound-sum-eps}]}\\
    &\le \log(T) \max_{h \in [H^\star]}\frac{t(h)}{\gamma(h)} \exp\left\{{O}\left( \sum_{k=2}^h \frac{T^{-1/4}}{\gamma(k) + k^{\beta} \Gamma(k)}\right) \right\} \quad \textrm{[As $B=\sqrt{T}\log T \log(\log(T)/\delta)$]}\\
\end{align*}
Now writing $g(h) = \exp\left\{{O}\left( \sum_{k=2}^h \frac{T^{-1/4}}{\gamma(k) + k^{\beta} \Gamma(k)}\right) \right\} $ we get the following inequality.
\begin{align}\label{eq:temp-result-2}
    \sum_{j=1}^{\log(T/B)} \max_{h \in [H^\star]} \frac{\hat{t}_j(h)}{\hat{\gamma}_j(h)} \le \log(T) \max_{h \in [H^\star]}\frac{t(h)}{\gamma(h)} g(h)
\end{align}
\begin{align*}
    &\sum_{j=1}^{\log (T/B)} \frac{\hat{\Delta}_j \abs{B_j}}{\hat{\gamma}_j(H^\star + 1)} \hat{t}_j(H^\star + 1) = \sum_{j=1}^{\log (T/B)} \frac{\hat{\Gamma}_j(H^\star + 1) \abs{B_j}}{\hat{\gamma}_j(H^\star + 1)} \hat{t}_j(H^\star + 1) \\
    &\le \sum_{j=1}^{\log (T/B)} \left(\Delta^\star + 2\eps_j\sqrt{\log T}\right) \abs{B_j} \frac{\hat{t}_j(H^\star + 1)}{\hat{\gamma}_j(H^\star + 1)}  \quad \textrm{[Since $\Delta^\star = \Gamma(H^\star + 1)$, and using lemma~\ref{lem:bound-Gamma}]}\\
    &\le \sum_{j=1}^{\log(T/B)}\left(\Delta^\star + 2\eps_j \sqrt{\log T} \right) \abs{B_j} \sum_{j=1}^{\log(T/B)} \frac{\hat{t}_j(H^\star + 1)}{\hat{\gamma}_j(H^\star + 1)} \quad \textrm{[Since $\sum_i a_i b_i \le \sum_i a_i \sum_i b_i$ for $a_i, b_i \ge 0$.]} \\
\end{align*}
The first summation can be bounded as follows.
\begin{align*}
\sum_{j=1}^{\log(T/B)} \left(\Delta^\star + 2\eps_j\sqrt{\log T}\right) \abs{B_j} &= \Delta^\star T + 2\sqrt{\log T}\sum_{j} \eps_j \abs{B_j}\\ &= \Delta^\star T + O\left(\sqrt{\log T} \sqrt{\log (\log(T/B)/\delta)}\right) \sum_j \sqrt{B_j}\\&\le \Delta^\star T + O\left(\sqrt{T} \log T\log \log T\right) 
\end{align*}
The second summation can be bounded by $\log(T) \frac{t(H^\star +1)}{\gamma(H^\star + 1)} g(H^\star + 1)$ by following the same steps used to derive \cref{eq:temp-result-2}. This gives us the following bound on the first term in \cref{eq:reg-expr-est}.
\begin{align}
    \label{eq:bound-first-term}
    \sum_{j=1}^{\log (T/B)} \frac{\hat{\Delta}_j \abs{B_j}}{\hat{\gamma}_j(H^\star + 1)} \hat{t}_j(H^\star + 1) \le \Delta^\star T + O\left(\sqrt{T} \log T\log \log T\right) + \log(T) \frac{t(H^\star + 1)}{\gamma(H^\star + 1)}g(H^\star + 1)
\end{align}
Now we bound the second term in \cref{eq:reg-expr-est}.
\begin{align*}
    &\sum_{j=1}^{\log(T/B)} \max_{h \in [H^\star]} \hat{t}_j(h) \frac{\hat{\Gamma}_j(h+1)}{\hat{\gamma}_j(h)}\\ &\le \log(T) \sum_{j=1}^{\log(T/B)} \max_{h \in[H^\star]} \frac{t(h)}{\gamma(h)}g(h)\left( {\Gamma}(h+1) + 2\eps_j\sqrt{\log T}\right) \quad \textrm{[By inequality~\ref{eq:temp-result-2} and lemma~\ref{lem:bound-Gamma}]}\\
    &\le \log(T)\log(T/B) \max_{h \in[H^\star]} \frac{t(h)}{\gamma(h)}g(h)\Gamma(h+1)\left(1+ \frac{2\sqrt{\log T}}{\Gamma(h+1)} \sum_{j=1}^{\log(T/B)} \eps_j \right)\\
    &\le \log(T) \log(T/B) \max_{h \in[H^\star]} \frac{t(h)}{\gamma(h)}g(h)\Gamma(h+1) O\left( 1 + \frac{2\sqrt{\log T}}{\Gamma(h+1) }\sqrt{\frac{\log(\log(T/B)/\delta)}{B}}\right)\\
    &\le \log^2(T) \max_{h \in[H^\star]} \frac{t(h)}{\gamma(h)}g(h)\Gamma(h+1) O\left( 1 + \frac{T^{-1/4}}{\Gamma(h+1)} \right) \quad \textrm{[As $B = \sqrt{T} \log T \log(\log(T/B)/\delta)$]}\\
\end{align*}
Now we have the following bound on the second term in \cref{eq:reg-expr-est}.
\begin{align}
    &\sum_{j=1}^{\log(T/B)} \max_{h \in [H^\star]} \hat{t}_j(h) \frac{\hat{\Gamma}_j(h+1)}{\hat{\gamma}_j(h)} O\left(\sqrt{SA \abs{B_j}H^\star} \right) \nonumber \\
    &\le \sum_{j=1}^{\log(T/B)} \max_{h \in [H^\star]} \hat{t}_j(h) \frac{\hat{\Gamma}_j(h+1)}{\hat{\gamma}_j(h)} \sum_{j=1}^{\log(T/B)} O\left(\sqrt{SA \abs{B_j}H^\star} \right)\nonumber \\
    &\le \log^2(T) \max_{h \in[H^\star]} \frac{t(h)}{\gamma(h)}\Gamma(h+1) g(h) O\left( 1 + \frac{T^{-1/4}}{\Gamma(h+1)} \right) \sqrt{\log(T/B)} \sqrt{\sum_{j=1}^{\log(T/B)} SA\abs{B_j}H^\star}\nonumber \\
    &\le \log^{5/2}(T) \max_{h \in[H^\star]} \frac{t(h)}{\gamma(h)}g(h) \Gamma(h+1) \left( 1 + \frac{O(T^{-1/4})}{\Gamma(h+1)}\right) O\left(\sqrt{SATH^\star}\right) \label{eq:bound-second-term}
\end{align}

Following similar steps, we can also bound the fourth term in equation~\ref{eq:reg-expr-est}.
\begin{align}
&\sum_{j=1}^{\log(T/B)} \max_{h \in [H^\star]} \hat{t}_j(h) \frac{\hat{\gamma}_j(h+1)}{\hat{\gamma}_j(h)} O\left(\sqrt{\abs{B_j} H^\star \log(\abs{B_j}H^\star/\delta)} \right)\nonumber\\
&\le \log^{5/2}(T) \max_{h \in [H^\star]}\frac{t(h)}{\gamma(h)}\gamma(h+1)g(h)\left(1 + \frac{O( T^{-1/4})}{\gamma(h+1)} \right) O\left(\sqrt{TH^\star \log(TH^\star/\delta)} \right)\label{eq:bound-fourth-term}
\end{align}
We now consider the third term in equation~\ref{eq:reg-expr-est}.
\begin{align*}
    &\sum_{j=1}^{\log(T/B)} \max_{h \in [H^\star]} \hat{t_j(h)} \frac{(\hat{\Gamma}_j(h+1))^2}{\hat{\gamma}_j(h) \hat{\gamma}_j(h+1)}\\
    &\le \log(T) \sum_{j=1}^{\log(T/B)} \max_{h \in [H^\star]} \frac{t(h)}{\gamma(h)}g(h) \left(\frac{\Gamma(h+1)}{\gamma(h+1)}+\frac{2\eps_j \sqrt{\log T}}{\gamma(h+1)} \right) \left(\Gamma(h+1) + 2\eps_j \sqrt{\log T} \right)\\
    &\textrm{[By inequality ~\ref{eq:temp-result-2}, and lemma~\ref{lem:bound-Gamma}, and \ref{lem:ratio-inequality-reverse}]}\\
    &\le \log (T) \max_{h\in[H^\star]} t(h) \frac{(\Gamma(h+1))^2}{\gamma(h)\gamma(h+1)}g(h) \sum_{j=1}^{\log(T/B)} \left(1 + \frac{2\eps_j \sqrt{\log T}}{\Gamma(h+1)} \right) \left(1 + \frac{2\eps_j \sqrt{\log T}}{\Gamma(h+1)} \right)\\
    &\le \log(T) \log(T/B) \max_{h\in[H^\star]} t(h) \frac{(\Gamma(h+1))^2}{\gamma(h)\gamma(h+1)}g(h) \left(1 + \sum_j \frac{2\eps_j \sqrt{\log T}}{\Gamma(h+1)} \right)^2\\
    &\le \log(T) \log(T/B) \max_{h\in[H^\star]} t(h) \frac{(\Gamma(h+1))^2}{\gamma(h)\gamma(h+1)}g(h) \left(1 + \frac{2\sqrt{\log T}}{\Gamma(h+1)} \sqrt{\frac{\log(\log(T/B)/\delta)}{B}} \right)^2\\
    &\le \log^2(T) \max_{h\in[H^\star]} t(h) \frac{(\Gamma(h+1))^2}{\gamma(h)\gamma(h+1)}g(h)  \left(1 + \frac{O(T^{-1/4})}{\Gamma(h+1)} \right)^2
\end{align*}
This gives the following bound on the third term.
\begin{align}
    &\sum_{j=1}^{\log(T/B)} \max_{h \in [H^\star]} \hat{t_j(h)} \frac{(\hat{\Gamma}_j(h+1))^2}{\hat{\gamma}_j(h) \hat{\gamma}_j(h+1)} O\left(S^2 A (H^\star)^{\beta} \log(\abs{B_j} H^\star) \right) \nonumber \\&\le \log^2(T) \log(T H^\star) \max_{h\in[H^\star]} t(h) \frac{(\Gamma(h+1))^2}{\gamma(h)\gamma(h+1)}g(h)  \left(1 + \frac{O(T^{-1/4})}{\Gamma(h+1)} \right)^2 O\left(S^2 A (H^\star)^{\beta} \right)\label{eq:bound-third-term}
\end{align}
Substituting bounds \ref{eq:bound-first-term}, \ref{eq:bound-second-term}, \ref{eq:bound-third-term}, and \ref{eq:bound-fourth-term} in eq.~\ref{eq:reg-expr-est} we get the following bound.
\begin{align*}
    &\sum_{j=1}^{\log (T/B)} \calR(H^\star, \hat{\Delta}_j, \abs{B_j}, \hbgamma_j) \le  \Delta^\star T + O\left(\sqrt{T} \log T\log \log T\right) + \log(T) \frac{t(H^\star + 1)}{\gamma(H^\star + 1)}g(H^\star + 1)\\
    &+ \log^{5/2}(T) \max_{h \in[H^\star]} \frac{t(h)}{\gamma(h)}g(h) \Gamma(h+1) \left( 1 + \frac{O(T^{-1/4})}{\Gamma(h+1)}\right) O\left(\sqrt{SATH^\star}\right)\\
&+  \log^2(T) \log(T H^\star) \max_{h\in[H^\star]} t(h) \frac{(\Gamma(h+1))^2}{\gamma(h)\gamma(h+1)}g(h)  \left(1 + \frac{O(T^{-1/4})}{\Gamma(h+1)} \right)^2 O\left(S^2 A (H^\star)^{\beta} \right)\\
&+  \log^{5/2}(T) \max_{h \in [H^\star]}\frac{t(h)}{\gamma(h)}\gamma(h+1)g(h)\left(1 + \frac{O( T^{-1/4})}{\gamma(h+1)} \right) O\left(\sqrt{TH^\star \log(TH^\star/\delta)} \right)
\end{align*}

\end{proof}


\begin{lemma}\label{lem:bound-Gamma}
Let $H_1,\ldots,H_n$ be iid according to a distribution $\calD$. For $h \in \set{1,\ldots,D}$ let $\hat{\Gamma}_n(h) = \sum_{u \ge h} \hat{\gamma}_n(h)$. Then with probability at least $1-\delta$ the following holds.
$$
\max_{h \in \set{1,\ldots,D}}\abs{\hat{\Gamma}_n(h) - \Gamma(h)} \le \sqrt{\frac{\log D + \log(1/\delta)}{n}}
$$
\end{lemma}
\begin{proof}
We will prove this statement by first constructing a class of functions, and proving uniform convergence over this class. In particular the class consists of a finite number of multi-valued functions, and we will use the bounded Natarajan dimension~\cite{Nat89} of this class to derive uniform convergence guarantees.

Let $h_v: \Nm \rightarrow \Nm$ be defined by $f_v(x) = \max\{0, x-v \}$.
Consider the class of functions $\calH = \set{f_v: v \in \set{1,2,\ldots,D}}$. Let $H_1,\ldots,H_n$ be $n$ random variables drawn i.i.d. from some distribution $\calD$. Then for any $h \in \set{1,\ldots,D}$ we have
\begin{align*}
    \frac{1}{n}\sum_{i=1}^n f_h(H_i) &= \frac{1}{n} \sum_{i=1}^n \max\{H_i - h, 0\} = \frac{1}{n} \sum_{i=1}^n \sum_{u=h}^\infty \one\set{H_i > u} \\
    &= \frac{1}{n} \sum_{i=1}^n \sum_{u=h}^\infty \left(1 - \one\set{H_i \le h} \right) = \sum_{u=h}^\infty \frac{1}{n} \sum_{i=1}^n \left(1 - \one\set{H_i \le u}\right) \\
    &= \sum_{u=h}^\infty \left(1 - \hat{F}_n(u) \right) = \sum_{u=h}^\infty \hat{\gamma}_n(u) =\hat{\Gamma}_n(h)
\end{align*}
Similarly one can show that $\E_{H \sim \calD}[f_h(H)] = \Gamma(h)$. Then by the multiclass fundamental theorem~\cite{SSS14} the following result holds as long as $n \ge \frac{\textrm{Ndim}(\calH) \log(D) + \log(1/\delta)}{\eps^2}$.
$$
\Pr\left(\max_{h \in \calH} \abs{\frac{1}{n}\sum_{i=1}^n f_h(H_i) - \E[f_h(H)]} > \eps\right) \le \delta
$$
This also implies that $\Pr\left(\max_{h \in [D]} \abs{\hat{\Gamma}_n(h) - \Gamma(h)} > \eps \right) \le \delta$. We now bound the term $\textrm{Ndim}(\calH)$ which is the Natarajan dimension of the class of functions $\calH$. In fact, we prove that $\textrm{Ndim}(\calH) = 1$.

Consider a set $\set{a,b} \subseteq [D]$ shattered by $\calH$. We will assume that $a < b$. This implies there exist two function $f_0$ and $f_1$ such that $f_0(x) \neq f_1(x)$ for $x\in \set{a,b}$. Moreover, for every $B \subseteq \set{a,b}$ there exists a function $h \in \calH$ such that $h(x) = f_0(x)$ for all $x \in B$ and $h(x) = f_1(x)$ for all $x \in B \setminus \set{a,b}$. 

We will use the following property of the class $\calH$. Any function $f_v \in \calH$ is characterized by a vector of the following form $(0,0,\ldots,0,1,2,\ldots,B-v)$ where the $i$-th entry is $f_v(i)$. Since all functions take value $0$ at point $1$, $a$ cannot be zero. Otherwise, either $f_0(a)$ or $f_1(a)$ must be non-zero, and there doesn't exist a function that takes that non-zero value at $a$. This contradicts the fact that the set $\set{a,b}$ is shattered by $\calH$.

Therefore, suppose $a \ge 2$. First observe that $f_0(b) - f_0(a) \le b-a$, as there does not exist any function $f \in \calH$ that jumps by more than $b-a$ as input changes from $a$ to $b$. By the same argument we must have $f_1(b) - f_1(a) \le b-a$. Moreover, as the set is shattered by $\calH$ there must exist a function $f\in\calH$ such that $f(a) = f_0(a)$ and $f(b) = f_1(b)$. This implies $f_1(b) - f_0(a) \le b-1$. Similarly we have $f_0(b) - f_1(a) \le b-a$.

We now consider two cases. First, $f_0(a) = 0$. Then $f_1(a) \ge 1$, and moreover $f_1(b) = f_1(a) + b-a \ge b-a+1$ as any function taking non-zero value at $a$ must increase by $1$ every step. However, this is a contradiction, as $f_1(b) - f_0(a) \ge b-a+1$.

For the second case, suppose $f_0(a) \ge 1$. In that case, $f_1(a) = f_0(a) + b-a$. We suppose $f_1(a) \ge f_0(a) + 1$. This is without loss of generality as the case $f_1(a) \le f_0(a) - 1$ can be covered by exchanging the roles of $f_0$ and $f_1$. Since $f_1(a) \ge 1$, we must have $f_1(b) = f_1(a) + b- a$. But this leads to a contradiction as $f_1(b) - f_0(a) = f_1(a) - f_0(a) + b-a \ge b-a+1$. Therefore, the set $\set{a,b}$ cannot be shattered by the function class $\calH$.
\end{proof}

\begin{lemma}\label{lem:ratio-inequality}
Suppose $\abs{\hat{\gamma}_j(h) - \gamma(h)} \le \eps_j$ and $\abs{\hat{\Gamma}_j(h) - \Gamma(h)} \le \eps_j(1 + \sqrt{\log T})$ then we have
$$
\abs{\frac{\hat{\gamma}_j(h)}{\hat{\Gamma}_j(h)} - \frac{\gamma(h)}{\Gamma(h)}} \le \frac{\eps_j(2 + \sqrt{\log T})}{\Gamma(h)}
$$
\end{lemma}
\begin{proof}
\begin{align*}
    \abs{\frac{\hat{\gamma}_j(h)}{\hat{\Gamma}_j(h)} - \frac{\gamma(h)}{\Gamma(h)}} &= \frac{\abs{\hat{\gamma}_j(h) \Gamma(h) - \gamma(h) \hat{\Gamma}_j(h)}}{\Gamma(h) \hat{\Gamma}_j(h)} \\
    &\le \frac{\hat{\gamma}_j(h) \abs{\Gamma(h) - \hat{\Gamma}_j(h) } + \hat{\Gamma}_j(h) \abs{\hat{\gamma}_j(h) - \gamma(h) }}{\Gamma(h) \hat{\Gamma}_j(h) }\\
    &\le \frac{\hat{\gamma}_j(h) \eps_j (1+\sqrt{\log T})}{\Gamma(h) \hat{\Gamma}_j(h)} + \frac{\eps_j}{\Gamma(h)} \le \frac{\eps_j(2 + \sqrt{\log T})}{\Gamma(h)}
\end{align*}
The last line uses $\hat{\gamma}_j(h) \le \hat{\Gamma}_j(h)$.
\end{proof}

\begin{lemma}\label{lem:ratio-inequality-reverse}
Suppose $\abs{\hat{\gamma}_j(h) - \gamma(h)} \le \eps_j$ and $\abs{\hat{\Gamma}_j(h) - \Gamma(h)} \le \eps_j(1 + \sqrt{\log T})$ then we have
$$
\abs{\frac{\hat{\Gamma}_j(h)}{\hat{\gamma}_j(h)} - \frac{\Gamma(h)}{\gamma(h)}} \le \frac{\eps_j(1 + \sqrt{\log T})}{\gamma(h)}
$$
\end{lemma}
\begin{proof}
\begin{align*}
    \abs{\frac{\hat{\Gamma}_j(h)}{\hat{\gamma}_j(h)} - \frac{\Gamma(h)}{\gamma(h)}} &= \frac{\abs{\hat{\Gamma}_j(h) \gamma(h) - \Gamma(h) \hat{\gamma}_j(h)}}{\gamma(h) \hat{\gamma}_j(h)} \\
    &\le \frac{\hat{\Gamma}_j(h) \abs{\gamma(h) - \hat{\gamma}_j(h) } + \hat{\gamma}_j(h) \abs{\hat{\Gamma}_j(h) - \Gamma(h) }}{\gamma(h) \hat{\gamma}_j(h) }\\
    &\le \frac{\hat{\Gamma}_j(h)}{\hat{\gamma}_j(h)}\frac{ \eps_j }{\gamma(h) } + \frac{\eps_j(1+\sqrt{\log T})}{\gamma(h)} 
\end{align*}
This gives us the following bound.
\begin{align*}
    \frac{\hat{\Gamma}_j(h)}{\hat{\gamma}_j(h)} - \frac{\Gamma(h)}{\gamma(h)} \le \frac{\hat{\Gamma}_j(h)}{\hat{\gamma}_j(h)}\frac{ \eps_j }{\gamma(h) } + \frac{\eps_j(1+\sqrt{\log T})}{\gamma(h)}
\end{align*}
Rearranging we get the following inequality.
\begin{align*}
    \frac{\hat{\Gamma}_j(h)}{\hat{\gamma}_j(h)} \le \frac{\Gamma(h)}{\gamma(h) + \eps_j} + \frac{\eps_j(1 + \sqrt{\log T})}{\gamma(h) + \eps_j} \le \frac{\Gamma(h)}{\gamma(h)} + \frac{\eps_j(1+\sqrt{\log T})}{\gamma(h)}
\end{align*}
\end{proof}

\subsection{Proof of Corollary~\ref{cor:geom-decay-est}}
\begin{proof}
For the geometric discount factor we have $\gamma(k) = \gamma^{k-1}$ and $\Gamma(k) = \gamma^{k-1}/(1-\gamma)$. We substitute $\beta = 3/2$ and get the following bound on $g(h)$.
\begin{align*}
    g(h) &= \exp\left\{O\left(\sum_{k=2}^h \frac{T^{-1/4}}{\gamma(k) + k^{\beta} \Gamma(k)} \right) \right\} \\&= \exp\left\{O\left(\sum_{k=2}^h \frac{1}{\gamma^{k-1}}\frac{T^{-1/4}}{1+ k^{3/2}/(1-\gamma)} \right) \right\} \\
    &\le \exp\left\{O\left(\frac{T^{-1/4}}{\gamma^{h-1}} \sum_{k=2}^h \frac{1}{1+k^{3/2}/(1-\gamma)} \right) \right\}\\
    &\le \exp\left\{O\left(\frac{T^{-1/4}}{\gamma^{h-1}} \int_{k=2}^h \frac{(1-\gamma)dk}{k^{3/2}} \right) \right\} \le\exp\left\{O\left( \frac{T^{-1/4}(1-\gamma)h^{-1/2}}{\gamma^{h-1}}\right) \right\}
\end{align*}

In corollary \ref{cor:exp-decay} we showed that $\frac{t(h)}{\gamma(h)} \le e^{1-\gamma}$ for any $h$. We also substitute $H^\star = \frac{\log T}{2\log(1/\gamma)}$. This gives us the following bounds.
\begin{align}
    &\max_{h \in [H^\star]}  \frac{t(h)}{\gamma(h)}g(h) \Gamma(h+1) \left( 1 + \frac{O(T^{-1/4})}{\Gamma(h+1)}\right) \nonumber\\
    &\le e^{1-\gamma} \max_{h \in [H^\star]} \exp\left\{O\left( \frac{T^{-1/4}(1-\gamma)h^{-1/2}}{\gamma^{h-1}}\right) \right\} \left(\frac{\gamma^h}{1-\gamma} + O(T^{-1/4}) \right) \nonumber\\
    &\le \frac{2\gamma e^{1-\gamma}}{1-\gamma} \max_{h \in [H^\star]} \exp\left\{O\left( \frac{T^{-1/4}(1-\gamma)h^{-1/2}}{\gamma^{h-1}}\right) \right\} \label{temp:result-1-exp}
\end{align}
Now we observe that the function $f(h) = \frac{h^{-1/2}}{\gamma^{h-1}}$ is a decreasing function of $h$ for $2h \log h > \log(1/\gamma)$. Otherwise $f$ is increasing in $h$. This implies $\max_{h \in [H^\star]} \frac{h^{-1/2}}{\gamma^{h-1}} \le \max\left\{1, T^{1/4} \sqrt{\frac{2 \log(1/\gamma)}{\log T}} \right\} \le  T^{1/4} \sqrt{\frac{2 \log(1/\gamma)}{\log T}} $ as long as $T\ge \frac{\log^2 T}{4(1-\gamma)^2}$. Substituting this bound in \cref{temp:result-1-exp} we get
\begin{align}
    &\max_{h \in [H^\star]}  \frac{t(h)}{\gamma(h)}g(h) \Gamma(h+1) \left( 1 + \frac{O(T^{-1/4})}{\Gamma(h+1)}\right) \nonumber\\
    &\le \frac{2\gamma e^{1-\gamma}}{1-\gamma} \exp\left\{O\left({\gamma}(1-\gamma) \sqrt{\frac{\log(1/\gamma)}{\log T}}\right) \right\} \le \frac{2\gamma e^{1-\gamma}}{1-\gamma} \exp\left\{O\left(\sqrt{\gamma}(1-\gamma) \right) \right\} = O\left( \frac{\gamma e^{1-\gamma}}{1-\gamma}\right) \label{eq:temp-first-term-est}
\end{align}
By a similar argument we can prove the following bounds.
\begin{align}\label{eq:temp-third-term-est}
    \max_{h \in [H^\star]}  \frac{t(h)}{\gamma(h)}g(h) \gamma(h+1) \left( 1 + \frac{O(T^{-1/4})}{\gamma(h+1)}\right) \le O\left( \frac{\gamma e^{1-\gamma}}{1-\gamma}\right)
\end{align}
\begin{align}
    &\max_{h \in [H^\star]} t(h) \frac{(\Gamma(h+1))^2}{\gamma(h)\gamma(h+1)}g(h)  \left( 1 + \frac{O(T^{-1/4})}{\Gamma(h+1)}\right)^2 \nonumber
    \\&\le \max_{h \in [H^\star]} e^{1-\gamma} \left(\frac{\gamma^{h/2}}{1-\gamma} + \gamma^{h/2}O(T^{-1/4}) \right)^2 g(h) = O\left(\frac{\gamma e^{1-\gamma}}{(1-\gamma)^2} \right) \label{eq:temp-second-term-est}
\end{align}
Finally, we bound the remaining term in the regret expression from theorem~\ref{thm:bound-alg-est}: $\min_{L\in [T]} T \Gamma(L+1) + 2L \log(T) \sqrt{T}$. In particular we substitute $L = \frac{\log T}{2 \log(1/\gamma)}$ to get the following bound on this term.
\begin{align}\label{eq:temp-final-regret-term}
    \frac{\sqrt{T}}{1-\gamma} + \frac{\sqrt{T}\log^2 T}{\log(1/\gamma)} = O\left(\frac{\sqrt{T}\log^2 T}{1-\gamma} \right)  
\end{align}
Note that this choice of $L$ requires $L \le T$ which is satisfied as long as $T \ge O(1/(1-\gamma)^2)$. Substituting bounds \ref{eq:temp-final-regret-term}, \ref{eq:temp-second-term-est}, \ref{eq:temp-third-term-est}, and \ref{eq:temp-first-term-est} in the regret expression from theorem~\ref{thm:bound-alg-est} we get the following bound on regret.
\begin{align*}
    \reg(\pi;\bgamma) \le &O\left(\frac{\sqrt{T}\log^2 T}{1-\gamma} \right) + O\left(\sqrt{T} \log T \log \log T \right)
    + \log^{5/2}(T) O\left(\frac{e^{1-\gamma}}{1-\gamma}\sqrt{SAT\frac{\log T}{\log(1/\gamma)}} \right) \\
    &+ \log^2(T) \log\left(\frac{T \log T}{\log(1/\gamma)} \right) O \left(\frac{e^{1-\gamma}}{(1-\gamma)^2} S^2 A \left(\frac{\log T}{\log(1/\gamma)} \right)^{3/2} \right) \\
    &\le O\left(\frac{\sqrt{T}\log^2 T}{1-\gamma} \right) + O\left( \frac{\sqrt{SAT}\log^3 T}{(1-\gamma)^{3/2}}\right) + O\left(\frac{S^2 A \log^{9/2} T}{(1-\gamma)^{7/2}} \right)\\
    &\le O\left( \frac{\sqrt{SAT}\log^3 T}{(1-\gamma)^{3/2}}\right) \quad \textrm{[If $T/\log^3 T \ge S^3 A/(1-\gamma)^4$]}
\end{align*}
\end{proof}

\subsection{Proof of Corollary~\ref{cor:poly-decay-est}}
\begin{proof}
For the polynomial discount factor we have $\gamma(k) = k^{-p}$ and $\Gamma(k) \in \left[\frac{k^{-p+1}}{p-1}, k^{-p} \left(1+\frac{k}{p-1} \right)\right]$. We substitute $\beta = p-1$ and get the following bound on $g(h)$.
\begin{align*}
    g(h) &= \exp\left\{O\left(\sum_{k=2}^h \frac{T^{-1/4}}{\gamma(k) + k^{\beta} \Gamma(k)} \right) \right\} 
    = \exp\left\{O\left(\sum_{k=2}^h \frac{T^{-1/4}}{k^{-p} + 1/(p-1)} \right) \right\} \\
    &\le \exp\left\{O\left({T^{-1/4}} \sum_{k=2}^h \frac{k^p}{1+k^{p}/(p-1)} \right) \right\}
    \le \exp\left\{O\left({T^{-1/4}}  \frac{h^{p+1}}{1+h^{p}/(p-1)} \right) \right\}
\end{align*}
The last inequality follows because the function $k^p / (1 + k^p / (p-1))$ is an increasing function of $k$. In corollary~\ref{cor:poly-decay} we showed that $\frac{t(h)}{\gamma(h)}\le e$ for any $h$ and $p > 1$. We also substitute $H^\star = T^{1/2p} $. This gives us the following bounds.
\begin{align*}
    &\max_{h \in [H^\star]} \frac{t(h)}{\gamma(h)}g(h) \Gamma(h+1) \left( 1+\frac{O(T^{-1/4})}{\Gamma(h+1)}\right)\\
    &\le e \max_{h \in [H^\star]} \exp\left\{O\left({T^{-1/4}}  \frac{h^{p+1}}{1+h^{p}/(p-1)} \right) \right\} \left( (h+1)^{-p}\left( 1 + \frac{h+1}{p-1}\right) + O(T^{-1/4}) \right)\\
    &\le e \left( 1 + \frac{2^{-(p-1)}}{p-1}\right) \max_{h \in [H^\star]} \exp\left\{O\left({T^{-1/4}}  \frac{h^{p+1}}{1+h^{p}/(p-1)} \right) \right\}
\end{align*}
For $p \ge 2$, $h^{p+1} / (1 + h^p  / (p-1)) \le 2h$. This gives the following bound on the term above.
\begin{align}\label{eq:poly-decay-est-bound-1}
\max_{h \in [H^\star]} \frac{t(h)}{\gamma(h)}g(h) \Gamma(h+1) \left( 1+\frac{O(T^{-1/4})}{\Gamma(h+1)}\right) \le 2e \exp\left\{ O(T^{-1/4} + T^{1/2p})\right\} \le 2e
\end{align}
By a similar argument we can prove the following bounds.
\begin{align}\label{eq:poly-decay-est-bound-2}
    \max_{h \in [H^\star]} \frac{t(h)}{\gamma(h)}g(h) \gamma(h+1) \left( 1+\frac{O(T^{-1/4})}{\gamma(h+1)}\right) \le 2e 
\end{align}
\begin{align}
       &\max_{h \in [H^\star]} t(h) \frac{(\Gamma(h+1))^2}{\gamma(h)\gamma(h+1)}g(h)  \left( 1 + \frac{O(T^{-1/4})}{\Gamma(h+1)}\right)^2 \nonumber
    \\&\le \max_{h \in [H^\star]} e (p-1) (h+1)^{-p}\left(1 + \frac{h+1}{p-1} \right)^2 \max_{h \in [H^\star]} g(h) \nonumber \\
    &\le 2e \quad \textrm{[As $p \ge 2$]} \label{eq:poly-decay-est-bound-3}
\end{align}
Finally, we bound the remaining term in the regret expression from theorem~\ref{thm:bound-alg-est}: $\min_{L \in [T]} T \Gamma(L+1) + 2L \log(T) \sqrt{T}$. We substitute $L = T^{1/2p}$ and get the following bound on the final term.
\begin{align}\label{eq:poly-decay-est-bound-4}
2 T^{1/2p} \log T + T\cdot T^{-1/2} \left( 1 + \frac{T^{1/2p}}{p-1}\right) = O\left( T^{\frac{1+p}{2p}}\right)
\end{align}
Substituting bounds \ref{eq:poly-decay-est-bound-1}, \ref{eq:poly-decay-est-bound-2}, \ref{eq:poly-decay-est-bound-3}, and \ref{eq:poly-decay-est-bound-4} in the regret expression from theorem~\ref{thm:bound-alg-est} we get the following bound on regret.
\begin{align*}
    \reg(\pi; \bgamma) &\le T^{\frac{p+1}{2p}} + T \left( T^{-1/2} + \frac{T^{-(1-p)/2p}}{p-1}\right) + O\left( \sqrt{T} \log T\right)\\
    &+ \log^{5/2}(T) O\left( \sqrt{SAT^{(2p+1)/2p}}\right) + \log^3 (T)\cdot O\left( S^2 A T^{(p-1)/2p}\right) \\&+ \log^{5/2}(T) \cdot O\left( \sqrt{T^{(2p+1)/2p} \log T }\right)
\end{align*}
If $T > (S^{3/2} A^{1/2})^p$ then the first and the fourth term dominates and regret is at most $O(\sqrt{SA}T^{(p+1)/2p})$. On the other hand, if $T < (S^{3/2} A^{1/2})^p$ then the fifth term dominates and regret is at most $\tilde{O}\left( S^2 A T^{(p-1)/2p}\right)$.
\end{proof}

\end{document}